\documentclass{svjour3}  
\smartqed  
\usepackage{graphicx}
\usepackage{bbold}
\usepackage{optidef}

\usepackage{amsthm,amsmath}
\usepackage[ruled,vlined]{algorithm2e}
\theoremstyle{definition}
\newtheorem{defn}{Definition}
\theoremstyle{plain}
\newtheorem{thm}{Theorem}
\newtheorem{lem}[thm]{Lemma}

\usepackage[utf8]{inputenc}
\usepackage[T1]{fontenc}
\usepackage{optidef}
\usepackage{hyperref}
\usepackage{caption}
\usepackage{subcaption}
\usepackage{comment}
\SetCommentSty{mycommfont}
\hypersetup{
    colorlinks=true,
    linkcolor=blue,
    filecolor=magenta,      
    urlcolor=black,
}
\newcommand{\argmax}{\mathop{\rm arg~max}\limits}
\newcommand{\argmin}{\mathop{\rm arg~min}\limits}
\DeclareUnicodeCharacter{3000}{} 
\usepackage[numbers]{natbib}
\usepackage{nicefrac} 
\usepackage{mathptmx}   
\usepackage{microtype} 
\usepackage{lipsum}
\usepackage{graphicx}
\usepackage{doi}
\usepackage{multirow}
\begin{document}

\title{Moreau-Yoshida Variational Transport: A General Framework For Solving Regularized Distributional Optimization Problems}

\author{Dai Hai Nguyen \and Tetsuya Sakurai}


\institute{Dai Hai Nguyen (Corresponding author) \at
              Department of Computer Science, University of Tsukuba, 1-1-1 Tennodai, Tsukuba, Ibaraki 305-8577, Japan \\
              \email{hai@cs.tsukuba.ac.jp or haidnguyen0909@gmail.com}           
           \and
          Tetsuya Sakurai \at
              Department of Computer Science, University of Tsukuba, 1-1-1 Tennodai, Tsukuba, Ibaraki 305-8577, Japan \\
              \email{sakurai@cs.tsukuba.ac.jp}}    

\date{Received: date / Accepted: date}
\maketitle
\begin{abstract}{We address a general optimization problem involving the minimization of a composite objective functional defined over a class of probability distributions. The objective function consists of two components: one  assumed to have a variational representation, and the other expressed in terms of the expectation operator of a possibly nonsmooth convex regularizer function. Such a regularized distributional optimization problem widely appears in machine learning and statistics, including proximal Monte-Carlo sampling, Bayesian inference, and generative modeling for regularized estimation and generation.

Our proposed method, named \textbf{M}oreau-\textbf{Y}oshida \textbf{V}ariational \textbf{T}ransport (\textbf{MYVT}), introduces a novel approach to tackle this regularized distributional optimization problem. First, as the name suggests, our method utilizes the Moreau-Yoshida envelope to provide a smooth approximation of the nonsmooth function in the objective. Second, we reformulate the approximate problem as a concave-convex saddle point problem by leveraging the variational representation. Subsequently, we develop an efficient primal-dual algorithm to approximate the saddle point. Furthermore, we provide theoretical analyses and present experimental results to showcase the effectiveness of the proposed method.}
\keywords{distributional optimization \and Moreau-Yoshida envelopes}
\end{abstract}

\section{Introduction}

Many tasks in machine learning and computational statistics are posed as distributional optimization problems, where the goal is to optimize a functional $F:\mathcal{P}_{2}(\mathcal{X})\rightarrow \mathbb{R}$ of probability distributions: $\min_{q\in \mathcal{P}_{2}(\mathcal{X})}F(q)$, where $\mathcal{P}_{2}(\mathcal{X})$ denotes the set of probability distributions defined on the domain $\mathcal{X}$ ($\subset \mathbb{R}^{d}$) with finite second-order moment. Examples of this formulation include many well-known problems such as Bayesian inference (e.g. variational autoencoder \cite{kingma2014stochastic}) and synthetic sample generation (e.g. generative adversarial networks \cite{goodfellow2020generative}). Basically, these models aim to approximate a target distribution by generating samples (or also called particles) in a manner that minimizes the dissimilarity between the empirical probability distribution obtained from the samples and the target distribution. The dissimilarity function typically involves measures such as Kullback-Leiber (KL) divergence, Jensen-Shanon (JS) divergence or Wasserstein distance in the optimal transport \cite{villani2021topics}.

In this paper, we consider a general class of distributional optimization problems with a composite objective functional of the following form:
\begin{align}
\label{eqn:regDP}
    \min_{q\in \mathcal{P}_{2}(\mathcal{X})}\left\{ G(q)\coloneqq F(q) + \alpha\mathbb{E}_{\textbf{x}\sim q}\left[g(\textbf{x})\right] \right\},
\end{align}
where $\mathbb{E}$ denotes the expectation operator and $g:\mathbb{R}^{d}\rightarrow \mathbb{R}$ denotes a possibly nonsmooth convex regularizer function for $\textbf{x}\in \mathcal{X}$ and $\alpha$ is a positive constant. Two choices of the function $g$ are the $l_{1}$-norm function $g(\textbf{x})=|\textbf{x}|$, which encourages sparse samples (with many elements equal 0), and one-dimensional total variation semi-norm $g(\textbf{x})=\sum_{i=2}^{d}|\textbf{x}_{i}-\textbf{x}_{i-1}|$, which encourages sparsity of the difference between nearby elements (i.e. local constancy of elements). Solving problem (\ref{eqn:regDP}) could be challenging due to the non-smoothness of the function $g$.

An example of the above formulation appears in the proximal Markov Chain Monte Carlo (MCMC) sampling \cite{durmus2018efficient,pereyra2016proximal}, which exploits convex analysis to obtain samples efficiently from log-concave density of the following form: $\text{exp}(-U(\textbf{x}))$, where $U(\textbf{x})=f(\textbf{x})+g(\textbf{x})$, $f:\mathbb{R}^{d}\rightarrow \mathbb{R}$ is a smooth convex function while $g$ is the nonsmooth convex function. By employing KL divergence to quantify the difference between the empirical probability distribution $q$  obtained from the current samples and the target distribution, the proximal MCMC sampling can be regarded as a specific instance of problem (\ref{eqn:regDP}).

We are particularly focused on devising an efficient algorithm for solving problem (\ref{eqn:regDP}) when the functional $F$ has a variational represenation in the following form:
\begin{align}
\label{eqn:varform}
    F(q) = \sup_{h\in\mathcal{H}}{\left\{ \mathbb{E}_{\textbf{x}\sim q}\left[ h(\textbf{x})\right]-F^{*}(h)\right\}},
\end{align}
where $\mathcal{H}$ is a class of square-integrable functions on $\mathcal{X}$ with respect to the Lebesgue measure and $F^{*}:\mathcal{H}\rightarrow \mathbb{R}$ is a convex conjugate functional of $F$. It is noted that the variational representation of $F$ involves a supremum over a function class $\mathcal{H}$. However, directly maximizing over an unrestricted function class is computationally challenging. To address it, we restrict $\mathcal{H}$ to a subset of the square-integrable functions such that maximization over $\mathcal{H}$ can be numerically solved, such as a class of deep neural networks or an Reproducing kernel Hilbert space (RKHS).
Furthermore, we assume throughout this paper that the maximum of $\mathbb{E}_{\textbf{x}\sim q}\left[ h(\textbf{x})\right]-F^{*}(h)$ over $\mathcal{H}$ exists, enabling the conversion of the supremum problem into a maximization problem.

For instance, when $F$ is the KL divergence, its variational representation is defined as:
\begin{align}
\label{eqn:varformKL}
    KL(q,\pi) = \sup_{h\in\mathcal{H}}{\left\{\mathbb{E}_{\textbf{x}\sim q}\left[ h(\textbf{x})\right]-\log\mathbb{E}_{\textbf{x}\sim \pi} \left[e^{h(\textbf{x})}\right]\right\}},
\end{align}
where $\pi$ denotes the target distribution. The solution to this problem can be estimated using samples from $q$ and $\pi$. In general, directly optimizing $F$ (problem (\ref{eqn:regDP}) with $\alpha=0$) can be achieved through an iterative algorithm called Wasserstein Gradient Descent \cite{santambrogio2015optimal}. This algorithm involves two steps in each iteration: 1) computing Wasserstein gradient of $F$ based on the current probability distribution and 2) performing an exponential mapping on $\mathcal{P}_{2}(\mathcal{X})$, the space of probability distributions. Variational transport (VT) \cite{liu2021infinite} proposes to minimize $F$ by approximating a probability distribution using a set of samples (or \textit{particles}). It leverages the variational representation of $F$ to update particles. Specifically, VT solves problem (\ref{eqn:varform}) by utilizing particles to approximate the Wasserstein gradient of $F$. The obtained solution is then used to push each particle in a specified direction. This process can be considered as a forward discretization of the Wasserstein gradient flow. 

However, when $\mathcal{X}$ is a constrained domain, VT may push the particles outside of the domain when following the direction given by the solution of (\ref{eqn:varform}). To address this issue, MirrorVT \cite{nguyen2023mirror} is introduced. The main idea behind mirrorVT is to map the particles from constrained domain (primal space) to an unconstrained domain (dual space), induced by a mirror map. Then, an approximate Wasserstein gradient descent is performed on the space of probability distributions defined over the dual space to update each particle, similar to VT. At the end of each iteration, the particles are mapped back to the original constrained domain.\\

\noindent
\textbf{Contributions}. We propose a novel method, named \textbf{M}oreau-\textbf{Y}oshida \textbf{V}ariational \textbf{T}ransport (\textbf{MYVT}), to tackle problem (\ref{eqn:regDP}). 
Our method deals with the non-smoothness of $g$ in the regularization term by employing the Moreau-Yoshida envelope \cite{rockafellar2009variational} to obtain a smooth approximation of $g$. By leveraging the variational representation of $F$, we reformulate the original problem as a concave-convex saddle point problem and develop an efficient primal-dual algorithm to approximate the saddle point. In contrast to the particle-based methods \cite{liu2021infinite, nguyen2023mirror}, which represent the probability distribution $q$ with a set of particles, MYVT employs a neural network to represent $q$ through \textit{reparameterization}. The network parameters are trained to optimize the objective of problem (\ref{eqn:regDP}). This approach addresses the issue of limited approximation capacity and the need for significant memory resources to store a large number of particles in particle-based methods. Furthermore, we provide theoretical analyses and present experimental results on synthetic and real-world datasets to showcase the effectiveness of the proposed method. 

\section{Related Works}
Our work is related to the research on methods for sampling distributions with nonsmooth convex composite potentials, which involve the sum of a continuously differentiable function and a possibly nonsmooth function. In particular, \cite{durmus2018efficient,pereyra2016proximal} introduced the proximal MCMC algorithm for quantifying uncertainty in Bayesian imaging application. They focused on the total-variation semi-norm \cite{rudin1992nonlinear} and $l_{1}$-norm \cite{park2008bayesian}, which encourage the generation of parameter estimates with specific structural properties. To handle the non-smoothness of the penalties, they utilized the Moreau-Yoshida envelope to obtain a smooth approximation to the total-variation semi-norm and $l_{1}$-norm. The Langevin algorithm was then employed to generate samples from the smoothed posterior distribution.

Furthermore, our method builds on VT \cite{liu2021infinite}, which leverages the optimal transport framework and variational representation (\ref{eqn:varform}) of $F$ to solve problem (\ref{eqn:regDP}) without regularization via particle approximation. In each iteration, VT estimated the Wasserstein gradient of $F$ by solving a variational maximization problem associated with $F$ and the current particles. It then performed Wasserstein gradient descent by moving particles in a direction specified by the estimated Wasserstein gradient. The advantage of VT was its ability to optimize $F$ beyond commonly targeted KL divergence in MCMC sampling methods.

When $g$ represents the indicator function of a set, problem (\ref{eqn:regDP}) transforms into a constrained distributional optimization problem, where the probability distributions are defined over a constrained domain. Applying VT to this problem may lead to particles moving outside the constrained domain. MirrorVT \cite{nguyen2023mirror} addressed this issue by mapping particles from the constrained domain to the unconstrained one through a mirror map, which is inspired by the Mirror Descent algorithm originally designed for constrained convex optimization \cite{beck2003mirror}. 

\section{Preliminaries}
We first review notions from convex analysis essential for our proposed method, specifically Moreau-Yoshida envelopes and proximal operators. Then, we review some concepts in optimal transport and Wasserstein space, and also state some related properties. We lastly summarize VT.
\subsection{Moreau-Yoshida Envelopes and Proximal Operators}
\begin{defn}
    (Moreau-Yoshida envelope) Given a convex function $g:\mathbb{R}^{d}\rightarrow \mathbb{R}$ and a positive scaling parameter $\lambda$, the Moreau-Yoshida envelope of $g$, denoted by $g^{\lambda}$, is given by:
    \begin{align}
    \label{def:envelope}
    g^{\lambda}(\textbf{x})=\inf_{\textbf{y}\in \mathbb{R}^{d}}\left\{ g(\textbf{y})+\frac{1}{2\lambda}\lVert \textbf{x}-\textbf{y} \rVert^{2}\right\}
    \end{align}.
    
\end{defn}
\noindent
The infimum of (\ref{def:envelope}) is always uniquely attained and the minimizer defines the proximal map of $g$:
\begin{defn}
    \label{def:proxmap}
    (Proximal Map) The proximal map of $g$, denoted by $\texttt{prox}^{\lambda}_{g}$, is given by:
    \begin{align}
    \texttt{prox}^{\lambda}_{g}(\textbf{x})=\argmin_{\textbf{y}\in \mathbb{R}^{d}}\left\{ g(\textbf{y})+\frac{1}{2\lambda}\lVert \textbf{x}-\textbf{y} \rVert^{2}\right\}
    \end{align}.
\end{defn}
\noindent
The approximation $g^{\lambda}$ inherits the convexity of $g$ and is always continuously differentiable. In particular, $g^{\lambda}$ is gradient Lipschitz: for $x,y\in \mathbb{R}^{d}$,
\begin{align*}
    \lVert \nabla g^{\lambda}(\textbf{x})- \nabla g^{\lambda}(\textbf{y}) \rVert \leq \frac{1}{\lambda}\lVert \textbf{x}-\textbf{y} \rVert,
\end{align*}
where the gradient $\nabla g^{\lambda}(\textbf{x})$ is given by:
\begin{align}
    \nabla g^{\lambda}(\textbf{x})=\frac{1}{\lambda}\left( \textbf{x} - \texttt{prox}^{\lambda}_{g}(\textbf{x})\right).
\end{align}
Furthermore, by \cite[Theorem 1.25]{rockafellar2009variational}, $g^{\lambda}(\textbf{x})$ converges pointwise to $g(\textbf{x})$ as $\lambda$ tends to zero.
\subsection{Optimal transport and Wasserstein space}
Optimal transport \cite{villani2021topics} has received much attention in the machine learning community and has been shown to be an effective tool for comparing probability distributions in many applications \cite{nguyen2023linear,nguyen2021learning,NIPS2019_9539}.
Formally, given a measurable map $T:\mathcal{X}\rightarrow \mathcal{X}$ and $p\in \mathcal{P}_{2}(\mathcal{X})$, we say that $q$ is the \textit{push-forward measure} of $p$ under $T$, denoted by $q=T\sharp p$, if for every Borel set $E\subseteq \mathcal{X}$, $q(E)=p(T^{-1}(E))$. For any $p,q\in \mathcal{P}_{2}(\mathcal{X})$, the $2$-Wasserstein distance $\mathcal{W}_{2}(p,q)$ is defined as:
\begin{align*}
   \mathcal{W}_{2}^{2}(p, q)= \inf_{\pi \in \Pi(p,q)} \int_{\mathcal{X}\times \mathcal{X}} \lVert \textbf{x} - \textbf{x}^\prime\rVert_{2}^{2}\mathrm{d}\pi(\textbf{x},\textbf{x}^\prime),
\end{align*}
where $\Pi(p,q)$ is all probability measures on $\mathcal{X}\times \mathcal{X}$ whose two marginals are equal to $p$ and $q$, and $\lVert\cdot \rVert_{2}$ denotes the Euclidean norm. It is known that the metric space $(\mathcal{P}_{2}(\mathcal{X}), \mathcal{W}_{2})$, also known as Wasserstein space, is an infinite-dimensional geodesic space \cite[Definition 6.4]{villani2009optimal}. 
In particular, a curve on $\mathcal{P}_{2}(\mathcal{X})$ can be expressed as: $\rho:[0,1]\times \mathcal{X}\rightarrow [0,\infty)$, where $\rho(t,\cdot)\in \mathcal{P}_{2}(\mathcal{X})$ for all $t\in[0,1]$. A tangent vector at $p$ can be expressed as $\partial \rho/\partial t$ for some curve $\rho$ such that $\rho(0, \cdot)=p(\cdot)$. In addition, under certain regularity conditions, the following elliptical equation $\partial \rho/\partial t=- \texttt{div}(\rho \nabla u)$ admits a unique solution $u:\mathcal{X}\rightarrow \mathbb{R}$ \cite{denny2010unique}, where $\texttt{div}$ denotes the divergence operator on $\mathcal{X}$. Therefore, the tangent vector $\partial \rho/\partial t$ is uniquely identified with the vector-valued mapping $\nabla u$. Furthermore, we can endow the manifold $\mathcal{P}_{2}(\mathcal{X})$ with a Riemannian metric \textcolor{blue}{\cite[p.~250]{villani2021topics}}, as follows: for any $s_{1},s_{2}$ are two tangent vectors at $p$, let $u_{1},u_{2}:\mathcal{X}\rightarrow \mathbb{R}$ be the solutions to the following elliptic equations $s_{1}=- \texttt{div}(\rho \nabla u_{1})$ and $s_{2}=- \texttt{div}(\rho \nabla u_{2})$, respectively. The inner product between $s_{1}$ and $s_{2}$ is defined as follows:
\begin{equation*}
    \langle s_{1}, s_{2}\rangle_{p}=\int_{\mathcal{X}}\langle \nabla u_{1}(\textbf{x}),\nabla u_{2}(\textbf{x})\rangle p(\textbf{x})d\textbf{x}.
\end{equation*}

\begin{defn}
    \label{def:firstvariational}
    (The first variation of functional) 
    the first variation of $F$ evaluated at $p$, denoted by $\partial F(p)/\partial p:\mathcal{X}\rightarrow \mathbb{R}$, is  given as follows:
\begin{align*}
    \lim_{\epsilon\rightarrow 0}\frac{1}{\epsilon} \left( F(p+\epsilon \chi) - F(p) \right)=
    \int_{\mathcal{X}} \frac{\partial F(p)}{\partial p}(\textbf{x})\chi(\textbf{x})\mathrm{d}\textbf{x},
\end{align*}
for all $\chi=q-p$, where $q\in \mathcal{P}_{2}(\mathcal{X})$.
\end{defn}
The Wasserstein gradient of $F$, denoted by $\texttt{grad}F$, relates to the gradient of the first variation of $F$ via the following continuity equation:
\begin{align}
\label{eqn:grad}
    \texttt{grad}F(p)(\textbf{x})=-\texttt{div}\left(p(\textbf{x})\nabla\frac{\partial F(p)}{\partial p}(\textbf{x})\right),
\end{align}
for all $\textbf{x}\in \mathcal{X}$.

\begin{defn}
    \label{def:stronglyconvex}
    (Geodesically strong convexity) 
    If $F$ is geodesically $\mu$-strongly convex with respect to the $2$-Wasserstein distance, then for $\forall p$, $p^\prime \in \mathcal{P}_{2}(\mathcal{X})$, we have: 
\begin{align*}
     F(p^\prime) &\geq F(p) + \langle \texttt{grad}F(p), \texttt{Exp}_{p}^{-1}(p^\prime)\rangle_{p} + \frac{\mu}{2}\cdot \mathcal{W}^{2}_{2}(p^\prime,p),
\end{align*}
\end{defn}
\noindent
where $\texttt{Exp}_{p}$ denotes the exponential mapping, which specifies how to move $p$ along a tangent vector on $\mathcal{P}_{2}(\mathcal{X})$ and $\texttt{Exp}_{p}^{-1}$ denotes its inversion mapping, which maps a point on $\mathcal{P}_{2}(\mathcal{X})$ to a tangent vector. We refer the readers to \cite{santambrogio2015optimal} for more details.


\subsection{Variational Transport}
\label{subsection:vt}
To optimize $F$ defined on the unconstrained domain, we can utilize functional gradient descent with respect to the geodesic distance. This approach involves constructing a sequence of probability distributions $\left\{ q_{t}\right\}_{t\geq 1}$ in $\mathcal{P}_{2}(\mathcal{X})$ as follows:
\begin{align}
\label{eqn:functionalgd}
    q_{t+1}\leftarrow \texttt{Exp}_{q_{t}}[ -\eta_{t}\cdot \texttt{grad}F(q_{t})],
\end{align}
where $\eta_{t}$ is the step size at the $t$-th step. The VT algorithm \cite{liu2021infinite} solves the distributional optimization problem by approximating $q_{t}$ with an empirical measure $\Tilde{q}_{t}$ obtained from $N$ particles $\left\{\textbf{x} _{t,i} \right\}_{i\in \left[N\right]}$. VT assumes that $F$ can be expressed in the variational representation (\ref{eqn:varform}). One advantage of this formulation is that the Wasserstein gradient can be computed based on the solution $h^{*}_{t}$ of problem (\ref{eqn:varform}), which can be estimated using samples drawn from $q_{t}$. Specifically, it can be shown that $h^{*}_{t}=\partial F/\partial q_{t}$, representing the first variation of $F$ \cite[Proposition 3.1]{liu2021infinite}. Furthermore,  under the assumption that $\nabla h^{*}_{t}$ is $h$-Lipschitz continuous, it can be demonstrated that for any $\eta_{t}\in \left[ 0,1/h\right)$, the exponential mapping in (\ref{eqn:functionalgd}) is equivalent to the push-forward mapping defined by $h^{*}_{t}$: for $\textbf{x}_{t,i}\sim q_{t}$:

\begin{align}
    \textbf{x}_{t+1,i}\leftarrow \texttt{Exp}_{\textbf{x}_{t,i}}[ -\eta_{t}\cdot \nabla h^{*}_{t}(\textbf{x}_{t,i})],
\end{align}
where $\textbf{x}_{t+1,i}$ is the updated particle, which is drawn from $q_{t+1}$, and $\texttt{Exp}_{\textbf{x}}[\eta\cdot \nabla u]$ denotes the transportation map, which sends $\textbf{x}\in  \mathcal{X}$ to a point $\textbf{x}+\eta\cdot \nabla u \in \mathcal{X}$ \cite[Proposition 3.2]{liu2021infinite}.
In addition, VT estimates the solution $h^{*}_{t}$ by solving problem (\ref{eqn:varform}) using finite samples drawn from $q_{t}$. This is achieved through stochastic gradient descent on the domain $\mathcal{X}$:

\begin{align}
\label{eqn:varformmax}
     \Tilde{h_{t}^{*}}= \argmax_{h\in\Tilde{\mathcal{H}}}{\left\{\frac{1}{N}\sum_{i=1}^{N}h(\textbf{x}_{t,i})-F^{*}(h)\right\}},
\end{align}
where $\Tilde{\mathcal{H}}$ is a function class, which can be specified to be the following class of deep neural networks:
\begin{align}
    \label{eqn:nnclass}
    \Tilde{\mathcal{H}}=\left\{\Tilde{h}  \Bigm| \Tilde{h}(\textbf{x})=\frac{1}{\sqrt{n_{w}}}\sum_{i=1}^{n_{w}}b_{i}\sigma([\textbf{w}]_{i}^{T}\textbf{x})\right\},
\end{align}
where $n_{w}$ is the width of the neural networks, $[\textbf{w}]_{i}\in \mathbb{R}^{d}$, $\textbf{w}=([\textbf{w}]_{1},...,[\textbf{w}]_{n_{w}})^{T}\in \mathbb{R}^{n_{w}\times d}$ is the input weight, $\sigma$ denotes a smooth activation function, and $b_{i}\in \left\{ -1,1\right\}$. In each iteration, the weights $\textbf{w}$ are guaranteed to lie in the $l_{2}$-ball centered at the initial weights $\textbf{w}(0)$ with radius $r_{h}$ defined as $\mathcal{B}^{0}(r_{h})=\left\{\textbf{w}: \lVert\textbf{w} - \textbf{w}(0) \rVert_{2}\leq r_{h} \right\}$. This choice of neural network class facilitates the analysis of the gradient error induced by the difference between $h^{*}_{t}$ and $ \Tilde{h_{t}^{*}}$ \cite{liu2021infinite}.

It is important to note that VT is a particle-based algorithm for solving the optimization problem: $\min_{q\in \mathcal{P}_{2}(\mathcal{X})}F(q)$, and it can handle various applications in machine learning. When $F$ represents the KL divergence, the optimal solution to the variational form (\ref{eqn:varformKL}) is given by $h^{*}(\textbf{x})=\log q(\textbf{x})/\pi(\textbf{x})$. Thus, the update of a sample at $k$-th iteration is given by: $\textbf{x}_{k+1}=\textbf{x}_{k}+\eta\log \pi(\textbf{x}_{k}) - \eta \log q_{k}(\textbf{x}_{k})$, where $q_{k}$ represents the density of sample at $k$-th iteration, corresponding to Lavengin MCMC algorithm \cite{cheng2018underdamped}. When $F$ is the Jensen-Shanon divergence, as shown in \cite{nguyen2023mirror}, its variational form can be reformulated as follows:
\begin{align*}
     \sup_{h^\prime\in\mathcal{H}^\prime}{\left\{ \frac{1}{2}\mathbb{E}_{\textbf{x}\sim q}\left[ \log\left( 1-h^\prime(\textbf{x}) \right) \right]  + \frac{1}{2}\mathbb{E}_{\textbf{y}\sim \pi}\left[ \log(h^\prime(\textbf{x})) \right]\right\}},
\end{align*}
where $\mathcal{H}^\prime$ is the space of functions with outputs between 0 and 1. It is noted that this reformulated form aligns with the objective of learning the discriminator in GANs \cite{goodfellow2020generative}.
In our experiments, we compare the proposed method to VT with different forms of $F$. This implies that we compare it to existing methods such as MCMC and GANs.

\section{Moreau-Yoshida Variational Transport}
To optimize $G$ in problem (\ref{eqn:regDP}), we can employ functional gradient descent to construct a sequence of probability distributions $\left\{ q_{t}\right\}_{t\geq 1}$ in $\mathcal{P}_{2}(\mathcal{X})$ as follows:
\begin{align}
\label{eqn:functionalgd}
    q_{t+1}\leftarrow \texttt{Exp}_{q_{t}}[ -\eta_{t}\cdot \texttt{grad}G(q_{t})]
\end{align} 
To develop an implementable algorithm, we need to characterize the Wasserstein gradient $\texttt{grad}G(q)$, as defined in (\ref{eqn:grad}). However, we face two challenges. First, the gradient of the first variation of $G$ is undefined due to the non-smoothness of function $g$, so we cannot define the Wasserstein gradient $\texttt{grad}G(q)$. Second, even if $G$ is smooth, defining $\texttt{grad}G(q)$ is non-trivial compared to the gradient descent algorithm in Euclidean space, as estimating the first variation of $G$ is necessary. 
To address these challenges, we introduce Moreau-Yoshida Variational Transport (MYVT). For the non-smoothness of $g$, we utilize the Moreau-Yoshida envelope to provide a smooth approximation of $g$ in the objective. To estimate the first variation of the approximate objective, we assume that $F$ has the variational representation. We subsequently reformulate the approximate problem as a concave-convex saddle point problem and develop an efficient primal-dual algorithm to approximate the saddle point. Further details are provided in the following subsections.

\subsection{Moreau-Yoshida approximation of problem (\ref{eqn:regDP})}
To address the non-smoothness of function $g$, our approach is to replace $g$ with its envelope $g^{\lambda}$, which leads to the following smooth approximate distributional optimization problem:
\begin{align}
\label{eqn:smoothregDP}
    \min_{q\in \mathcal{P}_{2}(\mathcal{X})}\left\{ G^{\lambda}(q)\coloneqq F(q) + \alpha\mathbb{E}_{\textbf{x}\sim q}[g^{\lambda}(\textbf{x})] \right\}.
\end{align}
We denote $\pi$ and $\pi^{\lambda}$ as the optimal solutions of problems (\ref{eqn:regDP}) and (\ref{eqn:smoothregDP}), respectively.  The following theorem establishes a connection between the two solutions.

\begin{thm}
\label{thm:twosolutions} Given that $F(q)$ is geodesically $\mu$-strongly convex ($\mu$>0), the solution $\pi^{\lambda}$ converges to $\pi$ as $\lambda$ goes to 0 with respect to the 2-Wasserstein distance, i.e.
\begin{equation}
    \label{eqn:limW2}
    \lim_{\lambda\to 0}\mathcal{W}^{2}_{2}(\pi^{\lambda}, \pi)=0.
\end{equation}
If $g$ is Lipschitz, i.e., for all $\textbf{x},\textbf{y}\in \mathbb{R}^{d}$, $|g(\textbf{x})-g(\textbf{y})|\leq \lVert g \rVert_{\text{Lip}}\lVert \textbf{x}-\textbf{y}  \rVert$, then for all $\lambda>0$,
\begin{equation}
    \label{eqn:boundW2}
   \mathcal{W}^{2}_{2}(\pi^{\lambda}, \pi)\leq \frac{\alpha \lambda}{2\mu}\lVert g \rVert^{2}_{\text{Lip}}.
\end{equation}
\end{thm}
\noindent
The proof of Theorem \ref{thm:twosolutions} is postponed to Appendix \ref{appendix:twosolutions}. 

\subsection{Primal-Dual Approach to problem (\ref{eqn:smoothregDP})}


\noindent
The objective in problem (\ref{eqn:smoothregDP}) still poses a challenge as it covers the entire space of probability distributions, making it generally computationally intractable. To tackle this problem, a common approach is to employ specific parameterization forms for the distribution $q$ and optimize its parameters or approximate it using a set of particles, as discussed in \cite{liu2021infinite, nguyen2023mirror}. However, these approaches often have limitations in terms of approximation ability and memory resources they require. Taking inspiration from \cite{wang2016learning}, we propose an alternative method of implicitly representing $q$ using a neural network. In this approach, we generate $\textbf{x}_{\epsilon} \sim q$ by passing $\epsilon$ drawn from a simple distribution $p_{\epsilon}$ through a network, i.e. $\textbf{x}_{\epsilon}=V(\epsilon, \theta)$, where $\theta$ denotes the network parameters, which are iteratively adjusted to minimize the objective in problem (\ref{eqn:smoothregDP}). 
In the following theorem, we present an equivalent primal-dual view of problem (\ref{eqn:smoothregDP}) by utilizing the variational representation of $F$:

\begin{thm}
\label{thm:saddle} We can formulate problem (\ref{eqn:smoothregDP}) equivalently as:

\begin{equation}
    \label{eqn:saddle1}
    \max_{h\in\mathcal{H}} \mathbb{E}_{\epsilon\sim p_{\epsilon}}\left[ \min_{\textbf{x}_{\epsilon}\in\mathcal{X}} \left\{ h(\textbf{x}_{\epsilon})+ \alpha g^{\lambda}(\textbf{x}_{\epsilon}) \right\}  \right] -F^{*}(h).
\end{equation}
\end{thm}
\noindent
The primal-dual formulation (\ref{eqn:saddle1}) is derived by applying the variational representation of $F$ and interchangeability principle introduced in \cite{dai2017learning,rockafellar2009variational}. 
The detailed proof of Theorem \ref{thm:saddle} is postponed to Appendix \ref{appendix:saddle}.

Based on the finding of Theorem \ref{thm:saddle}, we can transition to handling the distribution $q$ for each local variable. Specifically, given $\epsilon\sim p_{\epsilon}$ and a fixed $h$, we can solve the following local optimization problem for $\textbf{x}_{\epsilon}$:

\begin{equation}
    \label{eqn:localopt}
    \textbf{x}^{*}_{\epsilon}=\argmin_{\textbf{x}_{\epsilon}\in\mathcal{X}}\left\{ h(\textbf{x}_{\epsilon}) + \alpha g^{\lambda}(\textbf{x}_{\epsilon})\right\}.
\end{equation}

The objective function of each problem (\ref{eqn:localopt}) is differentiable, so it is common to use the classical gradient descent algorithm to optimize it. To simplify the analysis, we first consider the gradient flow, which is essentially the limit of gradient descent when step size tends to zero. The following theorem establishes a connection between the optimization of the local variables and the gradient flow for optimizing the regularized functional in (\ref{eqn:smoothregDP}). This connection holds in the limit case when utilizing a specific form of $h(\textbf{x}_{\epsilon})$:

\begin{thm}
\label{thm:gradflow} Let consider the following gradient flow for the local optimization problem (\ref{eqn:localopt}):
\begin{equation*}
\label{eqn:continuityequation}
    \frac{\partial \textbf{x}_{t}}{\partial t}=-v_{t}(\textbf{x}_{t}), \textbf{x}_{0}=V(\epsilon, \theta),
\end{equation*}
where $v_{t}(\textbf{x})\coloneqq \nabla h^{*}_{t}(\textbf{x})+\frac{\alpha}{\lambda}\left( \textbf{x}- \texttt{prox}^{\lambda}_{g}(\textbf{x})\right)$, for all $\textbf{x}\in\mathcal{X}$,\\ $h^{*}_{t} = \argmax_{h\in\mathcal{H}}{\left\{\mathbb{E}_{\textbf{x}\sim q_{t}}\left[h(\textbf{x})\right]-F^{*}(h)\right\}}=\frac{\partial F}{\partial q}\left( q_{t}\right)$ (see Subsection \ref{subsection:vt}), and $q_{t}$ is the distribution of particles $\textbf{x}_{t}$. Then, $q_{t}$ follows the following Fokker-Planck equation:
\begin{equation}
\label{eqn:continuityequation}
    \frac{\partial q_{t}}{\partial t}=\texttt{div}\left(  q_{t} v_{t}\right).
\end{equation}

\noindent
This is the Wasserstein gradient flow of $G^{\lambda}(q)$ in the space of probability distributions with 2-Wasserstein metric. Suppose $F$ is geodesically $\mu$-strongly convex. Then, the convergence of $G^{\lambda}(q_{t})$ is as follows, for $t\geq 0$:
\begin{align}
    G^{\lambda}(q_{t})-G^{\lambda}(\pi^{\lambda})\leq 
    \exp(-2\mu t)( G^{\lambda}(q_{0})-G^{\lambda}(\pi^{\lambda})).
\end{align}
\end{thm}
\noindent
The proof of Theorem \ref{thm:gradflow} is postponed to Appendix \ref{appendix:flowsolution}. We observe that in limit case, when $F$ is geodesically convex, $q_{t}$ exponentially converges to the minimizer of $G^{\lambda}(q)$ as $t\rightarrow \infty$. Here, we emphasize two crucial points in our work. First, the Moreau-Yoshida approximation provides a smooth surrogate for $g$, making gradient computation feasible. More specifically, leveraging Theorem \ref{thm:saddle}, we transition to handling the optimization problem for each local variable. For this, we can apply gradient-based methods, which are designed for smooth, differentiable functions. Second, the smoothness of $g^{\lambda}$ guarantees the existence and uniqueness of the gradient flow $\partial \textbf{x}_{t}/\partial t=-v_{t}(\textbf{x}_{t})$, and we can derive the Fokker-Plank equation (\ref{eqn:continuityequation}).

To efficiently solve problem (\ref{eqn:localopt}), we can take advantage of the advancements in optimization literature. In this work, we will focus on utilizing gradient descent for its simplicity and effectiveness.
Specifically, in each iteration of our method, we draw a batch of random inputs $\left\{ \epsilon_{i}\right\}_{i=1}^{m}$, where $m$ is mini-batch size. We then calculate the corresponding outputs for these inputs, which are subsequently used as the initial particles:
\begin{equation*}
    \textbf{x}^{(0)}_{i}=V(\epsilon_{i}, \theta).
\end{equation*}
We perform $T$ steps of gradient updates to optimize problem (\ref{eqn:localopt}) for each particle, i.e., for $t=0,...,T-1$:
\begin{align}
\begin{split}
\label{eqn:updatex}
    &\Delta^{(t)}_{i}= \nabla_{\textbf{x}} h(\textbf{x}^{(t)}_{i})  + \frac{\alpha}{\lambda} (\textbf{x}^{(t)}_{i} - \texttt{prox}^{\lambda}_{g}(\textbf{x}^{(t)}_{i})),\\
    &\textbf{x}^{(t+1)}_{i} = \textbf{x}^{(t)}_{i} - \eta \Delta^{(t)}_{i}.
\end{split}
\end{align}
The particles obtained from the last update, denoted as $\textbf{x}^{(T)}_{i}$ ($i=1,...,m$), approximate the solutions of the local optimization problems and are utilized to estimate $h$ in problem (\ref{eqn:saddle1}).
Furthermore, the particles undergo updates over multiple steps to converge towards the minimizers of local optimization problems. Therefore, the parameters $\theta$ of $V$ need to be updated such that it outputs $\{\textbf{x}^{(t+1)}_{i} \}_{i=1}^{m}$ instead of $\{\textbf{x}^{(t)}_{i} \}_{i=1}^{m}$. In other words, we aim to update $\theta$ as follows:
\begin{equation}
\label{eqn:updatetheta}
    \theta^{(t+1)}\leftarrow \argmin_{\theta} \sum_{i=1}^{m}\lVert V(\epsilon_{i}, \theta) - \textbf{x}^{(t+1)}_{i}\rVert^{2}.
\end{equation}
As suggested in \cite{feng2017learning}, we can perform only one step of gradient descent as follows:
\begin{equation}
\label{eqn:fastupdatetheta}
    \theta^{(t+1)}\leftarrow \theta^{(t)} - \eta \sum_{i=1}^{m}\nabla_{\theta}V(\epsilon_{i}, \theta^{(t)})\Delta^{(t)}_{i}.
\end{equation}
While the update (\ref{eqn:fastupdatetheta}) is an approximation of (\ref{eqn:updatetheta}), it offers computational efficiency and has shown promising performance in our experiments. 

\subsection{Practical Implementation}
We are now ready to present our algorithm for solving problem (\ref{eqn:smoothregDP}). The components of our method can be easily parameterized by deep neural networks and their parameters are estimated by mini-batch gradient descent, ensuring its scalability for large-scale datasets. For instance, our method requires the initialization $\textbf{x}^{(0)}_{i}$, which is the output of the network $V(\epsilon_{i}, \theta)$, where $\theta$ is the parameters. The function $h$ is parameterized by another neural network with parameters $W$. Taking into account of above parameters, we have our proposed algorithm illustrated in Algorithm 1. We perform $K$ iterations. For each iteration, initial particles are obtained by drawing a mini-batch of $m$ random inputs and calculating their corresponding outputs through $V$, then we perform $T$ steps of updates for each  particle. The particles obtained from the last steps are used for estimating $h$ by performing $T^\prime$ steps of updates to optimize its parameters $W$.

\begin{algorithm}[H]
\label{alg:rvt}
\SetAlgoLined
\KwIn{Functional $F$, mini-batch size $m$, number of iterations $K$, number of steps $T$ to update $V(\cdot, \theta)$ (see (\ref{eqn:updatex}) and (\ref{eqn:fastupdatetheta})), number of steps $T^\prime$ to update $h_{W}$, step size $\eta$, scaling parameter $\lambda$ and regularization parameter $\alpha$.}
\KwOut{Networks $V(\cdot, \theta)$, $h_{W}(\cdot)$}
Randomly initialize $\theta$, $W$ (parameters of $V$ and $h$) \\
$k\leftarrow 0$\\
\While{$k< K$}{
    sample mini-batch $\left\{ \epsilon_{i}\right\}_{i=1}^{m}\sim p_{\epsilon}$ 
    
    compute $\textbf{x}^{(0)}_{i}=V(\epsilon_{i}, \theta)$, for $i=1,...,m$

    $t\leftarrow 0$

    \While{$t<T$}{
     $\Delta^{(t)}_{i}= \nabla_{x} h(\textbf{x}^{(t)}_{i})  + \frac{\alpha}{\lambda} (\textbf{x}^{(t)}_{i} - \texttt{prox}^{\lambda}_{g}(\textbf{x}^{(t)}_{i}))$
     
    update $\textbf{x}^{(t+1)}_{i} = \textbf{x}^{(t)}_{i} - \eta \Delta^{(t)}_{i}$, for $i=1,...,m$

    update $\theta\leftarrow \theta - \eta \sum_{i=1}^{m}\nabla_{\theta}V(\epsilon_{i}, \theta)\Delta^{(t)}_{i}$

    $t\leftarrow t+1$
    }
    $\textbf{x}^{*}_{i}\leftarrow \textbf{x}^{(T)}_{i}$, for $i=1,...,m$

     $t^\prime\leftarrow 0$

    \While{$t^\prime<T^\prime$}{

    update $W=W+\eta \left( \frac{1}{m}\sum_{i=1}^{m}\nabla_{W}h(\textbf{x}^{*}_{i}) - \nabla _{W}F^{*}(h_{W}) \right)$
    
    $t^\prime\leftarrow t^\prime+1$
    }
    
    $k\leftarrow k+1$
    }
 \caption{Moreau-Yoshida variational transport (MYVT)}
\end{algorithm}

\section{Numerical Experiments}

\begin{figure*}
    \centering
    \includegraphics[width=1.0\textwidth]{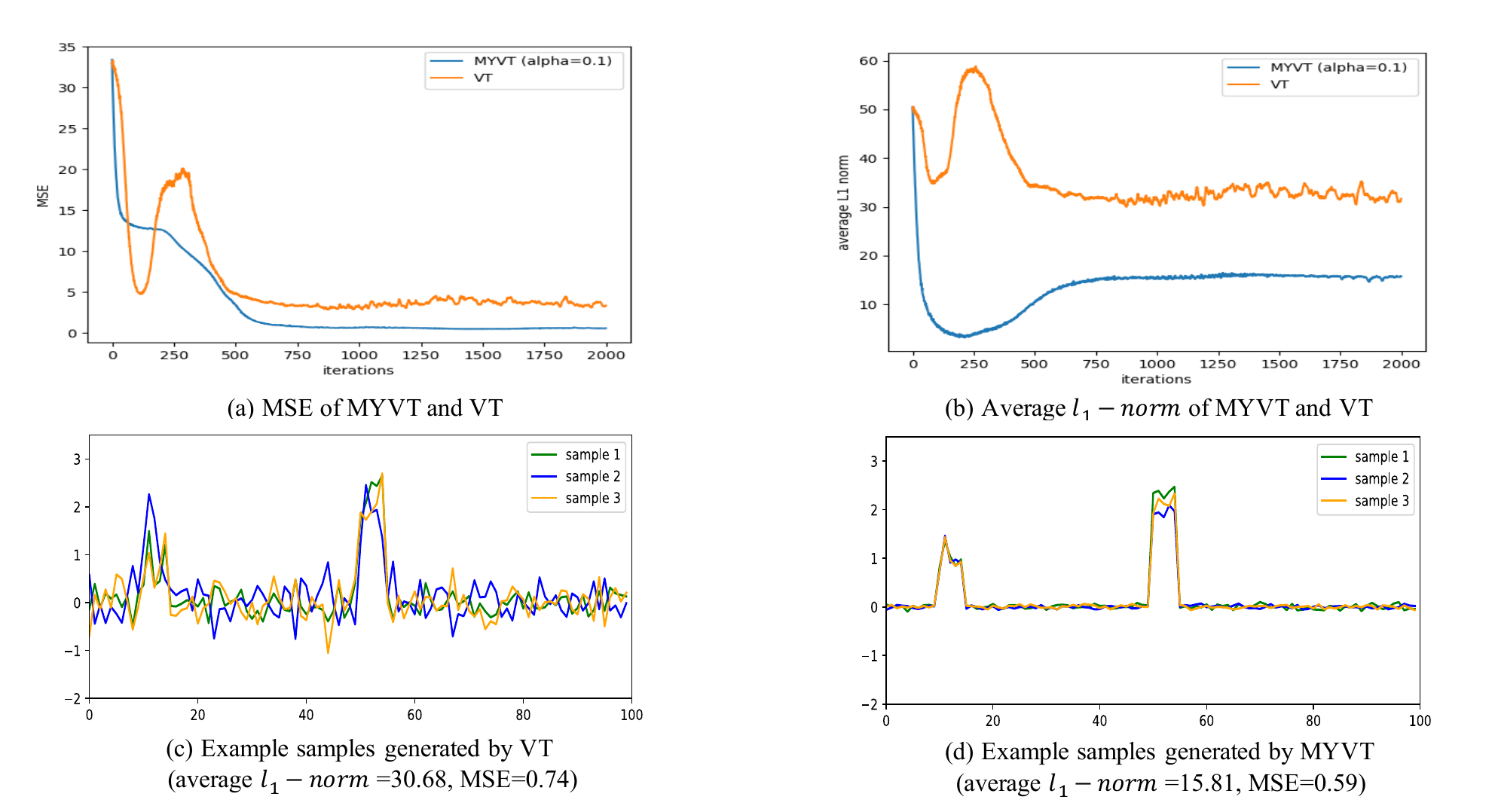}
    \caption{Comparison of MYVT ($\alpha=0.1$) and VT in terms of MSE and sparsity (average $l_{1}$-norm of generated samples) when $F$ is KL divergence i.e. $F(q)=KL(q, \pi)$. (a) MSE of MYVT and VT over 2000 iterations, (b) average $l_{1}$-norm over 2000 iterations, (c) three example samples generated by VT, (d) three example samples generated by MYVT.}
    \label{fig:sparsity}
\end{figure*}

In this section, we present numerical experiments on synthetic and real-world datasets to demonstrate the effectiveness of MYVT in noise robust signal/image generation. The code can be found at \url{https://github.com/haidnguyen0909/MYVT}.

\begin{figure*}
    \centering
    \includegraphics[width=1.0\textwidth]{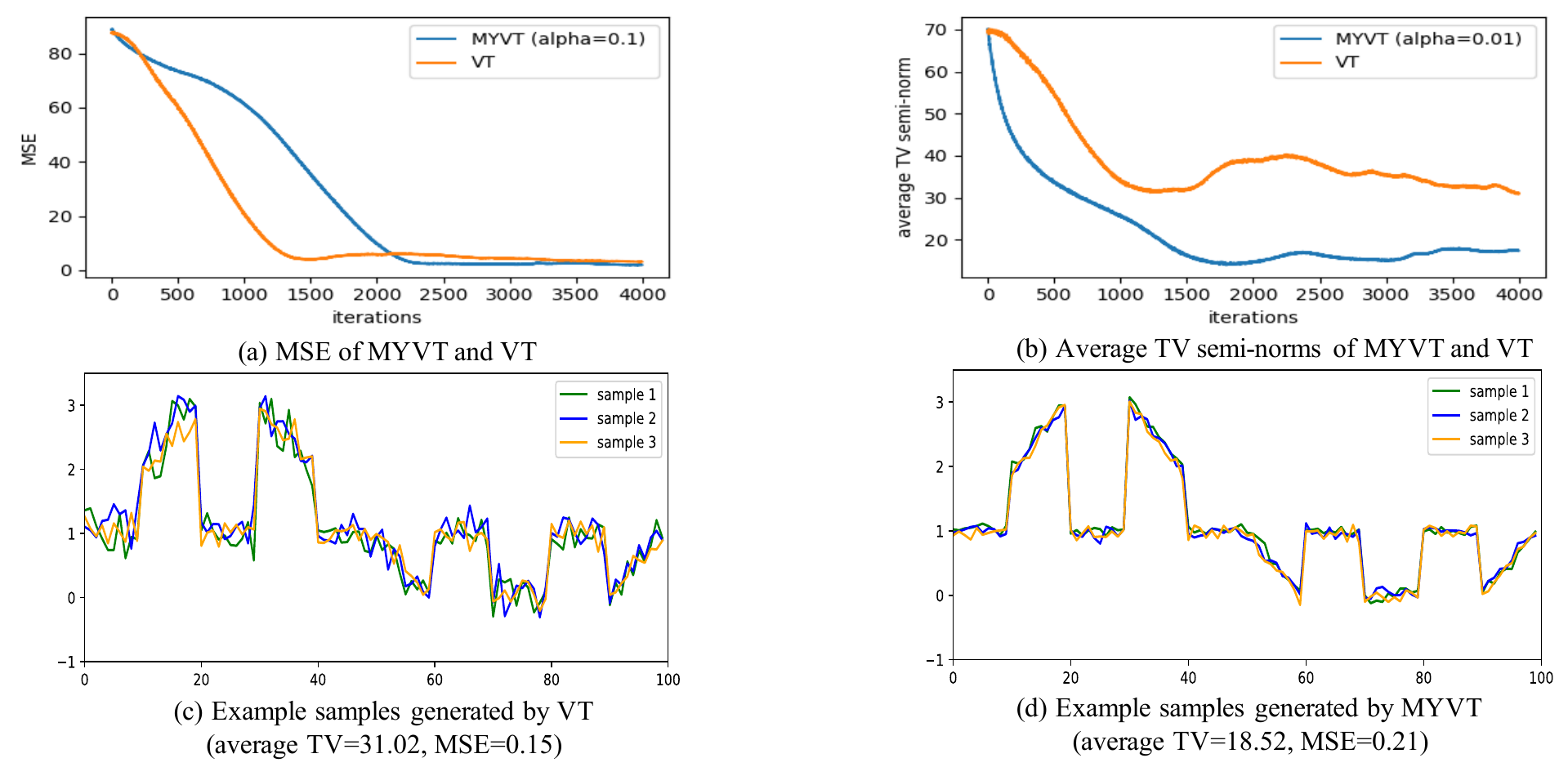}
    \caption{Evolution of example samples generated by VT and MYVT ($\alpha=100$) and VT in terms of MSE and smoothness (average TV semi-norm of generated samples) when $F$ is KL divergence i.e. $F(q)=KL(q, \pi)$. (a) MSE of MYVT and VT over 4000 iterations, (b) average TV semi-norm over 4000 iterations, (c) three example samples generated by VT, (d) three example samples generated by MYVT.}
    \label{fig:tv}
\end{figure*}

\subsection{Experiments with Synthetic Datasets}

\noindent
\textbf{Experimental settings}. We consider two case studies corresponding to two common choices
of the nonsmooth function $g$: (a) $l_{1}$-norm and (b) total variation (TV) semi-norm, which promotes sparsity of samples (i.e. few 
non-zero elements) and local constancy of elements (i.e. sparsity of the difference of nearby element), respectively. For the $l_{1}$-norm, the proximal map of $g(\textbf{x})=\lVert \textbf{x}\rVert_{1}$ is the well-known soft-thresholding operator $S_{\alpha}(\textbf{x})$ \cite{tibshirani1996regression}. For the TV semi-norm, the proximal map does not have a closed-form solution, but it can be efficiently estimated using the alternating direction method of multipliers (ADMM) \cite{marchuk1990splitting,wahlberg2012admm}. In our experiments, we perform 20 iterations of ADMM to estimate the proximal map. For the first case study, we design the truth $\textbf{z}\in \mathbb{R}^{100}$ to be sparse vector with only a few non-zero elements. For the second case study, we design $\textbf{z}$ to be a locally smooth vector. For each case study, we generate 500 examples $\left\{\textbf{y}_{i} \right\}_{i=1}^{500}$ by adding Gaussian noise with mean 0 and variance $0.04$ to the truth.
These settings allow to evaluate the performance of methods in recovering the underlying structure of the data in the presence of noise.

The generated examples are used to represent the target empirical distributions $\pi$ we aim to approximate. We set $F(q)=D(q, \pi)$, where $D$ represents a dissimilarity measure (KL or JS divergences in our experiments) between two probability distributions. 
We note that when $D$ represents the JS divergence, it can be shown that optimizing $JS(q, \pi)$ is equivalent to learning a GAN model. Indeed, as shown in \cite{nguyen2023mirror}, the variational representation of JS divergence is as follows:
\begin{align}
\label{eqn:gan1}
    JS(q, \pi) = \sup_{h\in\mathcal{H}^{c}}{\left\{ \mathbb{E}_{\textbf{x}\sim q}\left[ h(\textbf{x}) \right]   -JS^{*}(h)\right\}},
\end{align}
where $JS^{*}(h)=-\frac{1}{2}\mathbb{E}_{\textbf{x}\sim\pi}\left[ \log\left( 1-2e^{2h(\textbf{x})}\right) \right] - \frac{1}{2}\log 2$, and $\mathcal{H}^{c}$ is the space of function $h$ that satisfies: $h(\textbf{x})< 1/2\log(1/2)$ for all $\textbf{x}\in \mathcal{X}$. We introduce the following change of variable: $h^\prime(\textbf{x})=1-2e^{2h(\textbf{x})}$. It is easy to verify that $0<h^\prime(\textbf{x})<1$ for all $\textbf{x}\in \mathcal{X}$. Then, the variational representation of $JS$ divergence can be rewritten as:
\begin{align}
    \label{eqn:transform}
     \sup_{h^\prime\in\mathcal{H}^\prime}{\left\{ \frac{1}{2}\mathbb{E}_{\textbf{x}\sim q}\left[ \log\left( 1-h^\prime(\textbf{x}) \right) \right]  + \frac{1}{2}\mathbb{E}_{\textbf{x}\sim \pi}\left[ \log(h^\prime(\textbf{x})) \right]\right\}},
\end{align}
where $\mathcal{H}^\prime$ is the space of functions whose outputs are in between 0 and 1. As we have access to samples drawn from $q$ and $\pi$, we can estimate $h^\prime$ and then estimate $h$ using: $h(\textbf{x})=\frac{1}{2}\log\left( \frac{1-h^\prime(\textbf{x})}{2} \right)$ for all $\textbf{x}\in\mathcal{X}$. It can be verified that (\ref{eqn:transform}) corresponds to the objective function of learning the discriminator of GAN.\\

\noindent
\textbf{Comparing methods}. We compare MYVT and VT in the experiments. For VT, we represent $q$ using a set of particles and updating particles. For MYVT, we set $\alpha=0.1$ for both case studies. We evaluate and compare the quality of particles generated by the two methods using the following measures: a) mean squared error (MSE): the average squared difference between the generated samples and the truth $\textbf{z}$, (b) the average $l_{1}$-norm of generated samples for the first study case and (c) the average TV semi-norm of generated samples for the second study case. We parameterize $V$ using a neural network with four layers, each of which consists of a linear layer with 100 neurons, followed by an activation function.  We parameterize $h$ using another neural network with two layers, each of which has 100 neurons. The step sizes for VT and MYVT are fine-tuned and set to 0.01 and 0.0001, respectively. We run $K=2000$ and  $K=4000$ iterations for the first and the second case studies, respectively. For MYVT, we set $\lambda=0.0001$, $T=5$, $T^\prime=2$ and  step size $\eta=10^{-5}$ for all of experiments.\\

\noindent
\textbf{Results}. In the first case study, we compare MYVT and VT in terms of MSE and average $l_{1}$-norm when KL and JS divergence are used to represent $F$. As illustrated in Figure \ref{fig:sparsity}, when $F(q)=KL(q,\pi)$, both methods generate samples similar to the truth, indicated by the decreasing MSE values over 2000 iterations. However, MYVT maintains a significantly lower average $l_{1}$-norm than VT over iterations. Particularly, MYVT consistently produce samples with much lower average $l_{1}$-norm (around 15.81) compared to VT (around 30.68). This can be attributed to the effect of $g(\textbf{x})= \lVert \textbf{x} \rVert_{1}$ in problem (\ref{eqn:regDP}), promoting sparsity in the generated samples. Visually, samples generated by MYVT appear considerably sparser than those generated by VT (see Figures \ref{fig:sparsity}c and \ref{fig:sparsity}d). Similar effects are observed when considering $F(q)=JS(q,\pi)$, with figures are shown in Appendix \ref{appendix:syntheticexperiments}.

In the second case study, we compare MYVT and VT in terms of MSE and average TV semi-norm, when $F(q)=KL(q,\pi)$, as shown in Figure \ref{fig:tv}. Again both methods generate samples closely resembling the truth, evidenced by the significant decrease in MSE values over 4000 iterations (see Figure \ref{fig:tv}a), while MYVT maintains a significantly lower average TV semi-norm than VT over iterations (see Figure \ref{fig:tv}b). Particularly, the average TV semi-norm of samples generated by MYVT (around 18.52) is significantly smaller than that of samples generated by VT (around 31.02). Visually, samples generated by MYVT appear significantly smoother than those generated by VT (see Figures \ref{fig:tv}c and \ref{fig:tv}d). These results in both case studies demonstrate the regularization effect of problem (\ref{eqn:regDP}) on the generated samples and highlight the effectiveness of MYVT. Similar effects are observed when considering $F(q)=JS(q,\pi)$, with figures shown in Appendix \ref{appendix:syntheticexperiments}.

\begin{figure}[h]
\centering
    \begin{subfigure}[b]{.3\textwidth}
        \center
        \includegraphics[width=1.0\linewidth]{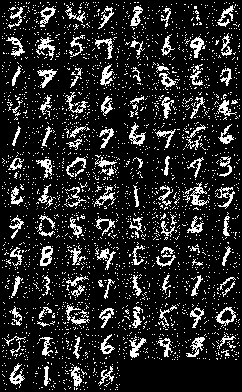}
        \caption{WGAN}
        \label{subfig1}
    \end{subfigure}
    \begin{subfigure}[b]{.3\textwidth}
        \center
        \includegraphics[width=1.0\linewidth]{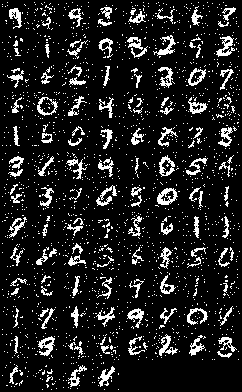}
        \caption{InfoGAN}
        \label{subfig1}
    \end{subfigure}
    \begin{subfigure}[b]{.3\textwidth}
        \center
        \includegraphics[width=1.0\linewidth]{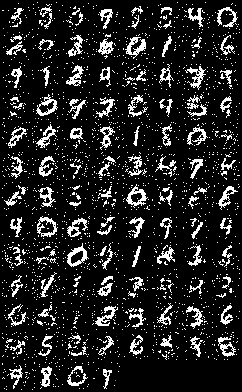}
        \caption{GAN}
        \label{subfig3}
    \end{subfigure}
    \centering
    
    \begin{subfigure}[b]{.3\textwidth}
        \center
        \includegraphics[width=1.0\linewidth]{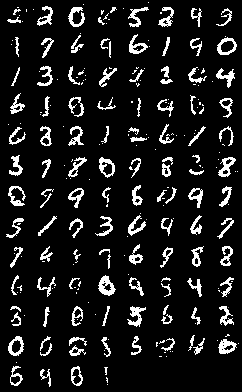}
        \caption{MYVT \\($\alpha=10^{-4}$)}
        \label{subfig:kl}
    \end{subfigure}
    \begin{subfigure}[b]{.3\textwidth}
        \center
        \includegraphics[width=1.0\linewidth]{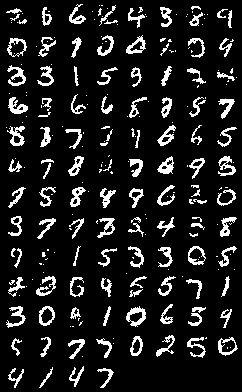}
        \caption{MYVT \\($\alpha=10^{-3}$)}
        \label{subfig:4}
    \end{subfigure}

    \caption{Synthetic images generated by WGAN (a), InfoGAN (b), GAN (c) and MYVT ((d) and (e)) trained on noisy training examples of MNIST at noise level $\sigma=0.5$.}
    \label{fig:syntheticimages0.5_mnist}
\end{figure}

\subsection{Experiments with Real-world Datasets}
In this subsection, we demonstrate the applicability of MYVT in noise-robust image generation. Specifically, we focus on the task of learning to generate clean images even when training images are corrupted by noise (noisy training examples). We utilize three common datasets: MNIST with 60,000 training images, FashionMNIST with 60,000 training images and Caltech101Silhouettes with 4,100 training images \footnote{Details of the datasets used in our experiments: MNIST (\url{https://www.kaggle.com/datasets/hojjatk/mnist-dataset}), FashionMNIST (\url{https://www.kaggle.com/datasets/zalando-research/fashionmnist}) and Caltech101Silhouettes (\url{https://people.cs.umass.edu/~marlin/data.shtml}).}. In our experimental setting, we assume that the generative models are trained on the noisy training examples. To simulate this scenario, we corrupt training images from these dataset by adding Gaussian noise to the clean training images, i.e., $\textbf{x} = \textbf{z} + \sigma \epsilon$, where $\textbf{z}\in\mathbb{R}^{h\times w}$ represents a clean image of width $w$ and height $h$ from the dataset; $\textbf{x}\in\mathbb{R}^{h\times w}$ represents a noisy image, and $\epsilon\sim \mathcal{N}(0, \textbf{I})$ is a Gaussian noise. For MYVT, we consider $F(q)=JS(q, \pi)$, where $\pi$ is represented by a collection of noisy images $\left\{\textbf{x}_{i} \right\}_{i=1}^{N}$. In our experiments, we compare MYVT to the following variants of GAN: original GAN \cite{goodfellow2020generative}, Wasserstein GAN (WGAN, \cite{arjovsky2017wasserstein}) and information-theoretic GAN (InfoGAN, \cite{10.5555/3157096.3157340}). For the regularization part, we consider two-dimensional TV semi-norm. Given an image $\textbf{x}\in \mathbb{R}^{h\times w}$, its two-dimensional TV semi-norm is defined as follows:
\begin{equation*}
TV(\textbf{x})=\sum_{i=1}^{h}\sum_{j=2}^{w}|\textbf{x}[i,j]-\textbf{x}[i,j-1]| +\sum_{j=1}^{w}\sum_{i=2}^{h}|\textbf{x}[i,j]-\textbf{x}[i-1,j]| .
\end{equation*}
Then, the optimization problem can be formulated as follows:
\begin{align*}
\label{eqn:regKL}
    \min_{q\in \mathcal{P}_{2}(\mathcal{X})} JS(q,\pi) + \alpha \mathbb{E}_{\textbf{x}\sim q}[\text{TV}(\textbf{x})] .
\end{align*}

\textcolor{blue}{
\begin{table}[]
    \centering
    \caption{grid search specs for MYVT.}
        \begin{tabular}{c c c}
        Settings & Description & Values \\
        \hline 
        $\alpha$ & regularization parameter & \{$10^{-5}$, $10^{-4}$, $10^{-3}$, $10^{-2}$\}\\
        $\eta$ & step size & \{$10^{-5}$, $10^{-4}$\}\\
        $\lambda$ & scaling parameter of Moreau-Yoshida envelope & \{$10^{-3}$, $10^{-2}$\}\\
        $K$ & number of iterations & \{$10^{5}$\}\\
        \texttt{nLayerh} & number of layers in $h^{\prime}$ & \{5\}\\
        \texttt{nLayerV} & number of layers in $V$ & \{5\}\\
        \texttt{optimh} & to use Adam or SGD for $h^{\prime}$ & Adam\\
        \texttt{optimV} & to use Adam or SGD for $V$ & Adam\\
        $T$ & number of steps to update $\theta$ of $h^{\prime}$ & \{$1$, $5$\}\\
        $T^{\prime}$ & number of steps to update $W$ of $V$ & \{$1$, $5$\}\\
        $m$ & mini-batch size & \{$100$\}\\
        \hline
        \# experiments & - & 64\\       
    \end{tabular}
    \label{tab:gridsearchspecs}
\end{table}
}

\noindent
\textbf{Experiment setting}. For MYVT, we conduct a grid search over a set of architectural choices and hyperparameters for each dataset. Formally, we specify the search grid in Table \ref{tab:gridsearchspecs}. We perform 64 experiments for each dataset to identify the optimal hyperparameter values for MYVT based on the quality of generated images. To quantitatively evaluate the generated images, we utilize the Frechet Inception distance (FID) and inception score (IS). We specify the best hyperparameter values as follows. We select different values for $\alpha$, including $10^{-4}$ or $10^{-3}$. We parameterize $V$ using a neural network with five layers, each containing 200 neurons. We parameterize $h^\prime$ using another neural network with five layers, each containing 200 neurons. We employ the ReLU activation function in the last layer to ensure the output of $h^\prime$ remains positive. The number of iterations $K$, mini-batch size $m$, $T$, $T^\prime$, scaling parameter $\lambda$ and the step size $\eta$ are set as $10^{5}$, 100, 1, 1, $10^{-3}$ and $10^{-5}$, respectively. For GAN models (including WGAN, InfoGAN and GAN), architecture choices and hyperparameter values, including $K$, $m$ and $\eta$, are kept consistent with MYVT to ensure a fair and meaningful comparison between the methods. The number of categories used in InfoGAN is set as 10. We use Adam \cite{kingma2014adam} for updating the neural network parameters. We train the models on the same machine with NVIDIA Quadro P5000 GPU with 16GB memory.\\


\begin{figure}
\centering
    \begin{subfigure}{.3\textwidth}
        \center
        \includegraphics[width=1.0\linewidth]{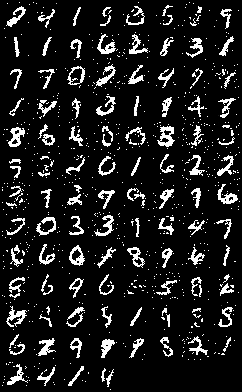}
        \caption{WGAN}
        \label{subfig:5}
    \end{subfigure}
    \begin{subfigure}{.3\textwidth}
        \center
        \includegraphics[width=1.0\linewidth]{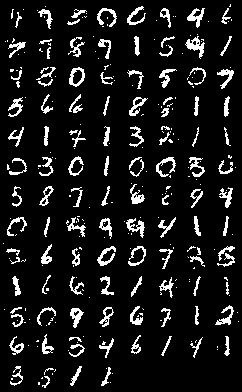}
        \caption{InfoGAN}
        \label{subfig:6}
    \end{subfigure}
    \begin{subfigure}{.3\textwidth}
        \center
        \includegraphics[width=1.0\linewidth]{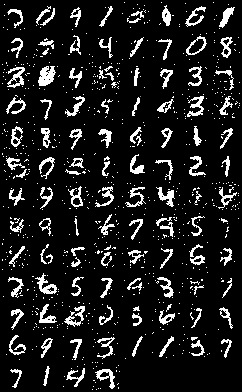}
        \caption{GAN}
        \label{subfig:7}
    \end{subfigure}

    \centering
    \begin{subfigure}{.3\textwidth}
        \centering
        \includegraphics[width=1.0\linewidth]{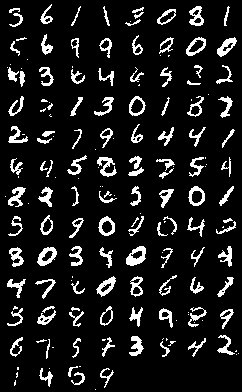}
        \caption{MYVT \\($\alpha=10^{-4}$)}
        \label{subfig:8}
    \end{subfigure}
    \begin{subfigure}{.3\textwidth}
        \center
        \includegraphics[width=1.0\linewidth]{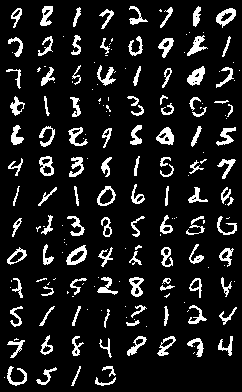}
        \caption{MYVT \\($\alpha=10^{-3}$)}
        \label{subfig:9}
    \end{subfigure}

    \caption{Synthetic images generated by WGAN (a), InfoGAN (b), GAN (c) and MYVT ((d) and (e)) trained on noisy training examples of MNIST at noise level $\sigma=0.1$.}
    \label{fig:syntheticimages0.1_mnist}
\end{figure}

\noindent
\textbf{Results}. We consider the synthetic images generated by WGAN, InfoGAN, GAN and MYVT (with $\alpha=10^{-4},10^{-3}$) at noise levels $\sigma=0.5$ and $0.1$. The evaluation is based on a comparison with real and clean images from the datasets and is performed by the pretrained InceptionV3 image classification model provided by the Torch framework.

Table \ref{tab:fidandis1} and Table \ref{tab:fidandis2} present the computed FID and IS for the synthetic images generated by models trained on MNIST, FashionMNIST and Caltech101Silhouettes. Lower FID values indicate better performance while higher IS values are preferred. Notably, MYVT consistently outperforms WGAN, InfoGAN and GAN models significantly in terms of both FID and IS across all datasets, noise levels, and $\alpha$ values. For instance, at $\sigma=0.5$, MYVT achieves the best FID values of 0.12, 0.72, 0.09, and the best IS values of  2.15, 3.01, 2.46 on MNIST, FashionMNIST, Caltech101Silhouettes, respectively. Similarly, at $\sigma=0.1$, MYVT attains the best FID values of 0.15, 0.39 and 0.07, and the best IS values of 2.68, 3.41, 2.45 on these datasets.

Figure \ref{fig:syntheticimages0.5_mnist} and Figure \ref{fig:syntheticimages0.1_mnist} visually illustrate the quality of synthetic images generated by WGAN, InfoGAN, GAN and MYVT trained on noisy images of MNIST at noise levels $\sigma=0.5$ and $0.1$, respectively. Similar visualizations for FashionMNIST and Caltech101Silhouettes can be found in Appendix \ref{syntheticimage}. The comparison highlights that MYVT effectively learns to generate clean synthetic images, outperforming other models, including WGAN, InfoGAN and GAN, which replicate images faithfully in presence of noisy training images. These results provide strong evidence for the efficacy of regularization and our proposed method, MYVT.

In terms of computational time, we compare MYVT to GAN. We observe that MYVT is almost twice as computationally expensive as GAN due to the additional overhead associated with evaluating the proximal map of $g$. In our experiments, $g$ is the total variation norm, requiring 20 iterations of ADMM to estimate its proximal map. To provide specifics, the total computational times required for GAN to train $10^{5}$ iterations on MNIST, FashionMNIST, and Caltech101Silhouettes are around 2.5, 2.3 and 0.25 hours, respectively. Conversely, the total computational times required for MYVT to train $10^{5}$ iterations on MNIST, FashionMNIST and Caltech101Silhouettes are around 6, 5 and 0.4 hours, respectively. These indicate that the complexity of MYVT partially contingent on the computation of the proximal maps.


\begin{table}[]
    \centering
    \caption{The Frechet Inception Distance (FID) and Inception Score (IS) for evaluating synthetic images generated by WGAN \cite{arjovsky2017wasserstein}, InfoGAN \cite{10.5555/3157096.3157340}, GAN \cite{goodfellow2020generative} and MYVT ($\alpha={10^{-3}, 10^{-4}}$) at noise level $\sigma = 0.5$. For FID, the lower the better, and for IS, the higher the better. The best scores are shown in bold.}
        \begin{tabular}{l c c c c c c}
        \hline 
        & \multicolumn{6}{c}{Datasets}\\
        \cline{2-7}
        \multirow{2}{*}{Models}& \multicolumn{2}{c}{MNIST} & \multicolumn{2}{c}{FashionMNIST} & \multicolumn{2}{c}{Caltech101Silhouettes}\\
        \cline{2-7}
        & FID & IS & FID & IS & FID & IS\\
        \hline
        WGAN \cite{arjovsky2017wasserstein}  & 9.74 & 1.23 & 4.72 & 2.54 & 2.46 & 1.87\\
        InfoGAN \cite{10.5555/3157096.3157340}  & 4.18 & 2.03 & 3.94 & 2.85 & 3.61 & 2.12\\
        GAN \cite{goodfellow2020generative}  & 10.12 & 1.09 & 4.51 & 2.81 & 3.75 & 2.06\\
        MYVT ($\alpha=10^{-3}$) & 0.15 & \textbf{2.15} & \textbf{0.72} & \textbf{3.01} & \textbf{0.09} & \textbf{2.46}\\
        MYVT ($\alpha=10^{-4}$) & \textbf{0.12} & 2.08 & 0.97 & 2.91 & 0.11 & 2.38\\
    \end{tabular}
    \label{tab:fidandis1}
\end{table}

\begin{table}[]
    \centering
    \caption{The FID and IS for evaluating synthetic images generated by WGAN \cite{arjovsky2017wasserstein}, InfoGAN \cite{10.5555/3157096.3157340}, GAN \cite{goodfellow2020generative} and MYVT ($\alpha={10^{-3}, 10^{-4}}$) at noise level $\sigma = 0.1$. For FID, the lower the better, and for IS, the higher the better. The best scores are shown in bold.}
        \begin{tabular}{l c c c c c c}
        \hline 
        & \multicolumn{6}{c}{Datasets}\\
        \cline{2-7}
        \multirow{2}{*}{Models}& \multicolumn{2}{c}{MNIST} & \multicolumn{2}{c}{FashionMNIST} & \multicolumn{2}{c}{Caltech101Silhouettes}\\
        \cline{2-7}
        & FID & IS & FID & IS & FID & IS\\
        \hline
        WGAN \cite{arjovsky2017wasserstein}  & 3.59 & 2.13 & 1.49 & 3.34 & 1.46 & 1.99\\
        InfoGAN \cite{10.5555/3157096.3157340}  & 2.29 & 2.54 & 1.45 & 3.29 & 0.65 & 2.17\\
        GAN \cite{goodfellow2020generative}  & 2.19 & 2.37 & 1.61 & 3.25 & 1.37 & 1.85\\
        MYVT ($\alpha=10^{-3}$) & \textbf{0.15} & 2.15 & \textbf{0.39} & 3.38 & \textbf{0.07} & 2.39\\
        MYVT ($\alpha=10^{-4}$) & 0.17 & \textbf{2.68} & 0.62 & \textbf{3.41} & \textbf{0.07} & \textbf{2.45}\\
    \end{tabular}
    \label{tab:fidandis2}
\end{table}

\section{Conclusion}
We have addressed the regularized distributional optimization problem with a composite objective composed of two functionals. The first one has the variational representation while the second one is expressed in terms of the expectation operator of a non-smooth convex regularizer function. We have introduced MYVT as a solution to this problem. Its key idea is to approximate the original problem using Moreau-Yoshida approximation and reformulate it as a concave-convex saddle point problem by leveraging the variational representation. In our future work we aim to develop more efficient algorithms for estimating the solutions of problem (\ref{eqn:varform}). Additionally, we plan to extend MYVT to handle other forms of objective functionals, which do not possess the variational representation. By exploring these directions, we aim to enhance the versatility and efficiency of MYVT and further advance the field of regularized distributional optimization.

\section{Ethics declarations}
\textbf{Fundings} Not applicable\\
\textbf{Conflict of interest} Not applicable\\
\textbf{Ethical approval} Not applicable\\
\textbf{Consent to participate} Not applicable\\
\textbf{Consent for publication} Not applicable\\
\textbf{Availability of data and material} Not applicable\\
\textbf{Code availability} Codes for experiments can be accessed through \url{https://github.com/haidnguyen0909/MYVT} after the acceptance of the paper.\\
\textbf{Authors' contributions} Not applicable\\
\bibliographystyle{plainnat}
\bibliography{mlj}  





\newpage
\appendix
\section{Appendix}
\subsection{Proof of Theorem \ref{thm:twosolutions}}.
\label{appendix:twosolutions}
\begin{proof}
By the definitions of the functionals $G(q)$ and $G^{\lambda}(q)$, we have:
\begin{equation}
    \label{appendix:eqn1}
    G^{\lambda}(\pi)-G(\pi)= \alpha\mathbb{E}_{\textbf{x}\sim\pi}[g^{\lambda}(\textbf{x})-g(\textbf{x})].
\end{equation}
We can decompose the left-hand side of (\ref{appendix:eqn1}) into three terms as follows:
\begin{align}
\begin{split}
    \label{appendix:eqn2}
    G^{\lambda}(\pi)-G(\pi)&= \underbrace{G^{\lambda}(\pi)-G^{\lambda}(\pi^{\lambda})}_{(a)} + \underbrace{G^{\lambda}(\pi^{\lambda})-G(\pi^{\lambda})}_{(b)}
    + \underbrace{G(\pi^{\lambda})  - G(\pi)}_{(c)}.
\end{split}
\end{align}
As $F(q)$ is geodesically $\mu$-strongly convex ($\mu$>0), it is easy to verify that $G(q)$ and $G^{\lambda}(q)$ are also geodesically $\mu$-strongly convex. Therefore, we obtain the following inequalities:
\begin{align}
\label{ineq:a}
\begin{split}
    \text{(a)} &\geq \langle \texttt{grad}G^{\lambda}(\pi^{\lambda}), \texttt{Exp}_{\pi^{\lambda}}^{-1}(\pi)\rangle_{\pi^{\lambda}} + \frac{\mu}{2}\mathcal{W}^{2}_{2}(\pi^{\lambda},\pi)
    = \frac{\mu}{2}\mathcal{W}^{2}_{2}(\pi^{\lambda},\pi),
\end{split}
\end{align}
\noindent
\begin{align}
\label{ineq:c}
\begin{split}
    \text{(c)} &\geq \langle \texttt{grad}G(\pi), \texttt{Exp}_{\pi}^{-1}(\pi^{\lambda})\rangle_{\pi} + \frac{\mu}{2}\mathcal{W}^{2}_{2}(\pi^{\lambda},\pi)
    = \frac{\mu}{2}\mathcal{W}^{2}_{2}(\pi^{\lambda},\pi).
\end{split}
\end{align}
These inequalities are obtained by the assumption that $\pi$ and $\pi^{\lambda}$ are the minimizers of $G(q)$ and $G^{\lambda}(q)$, respectively. Again, by the definition of $G(q)$ and $G^{\lambda}(q)$, we have:
\begin{equation}
\label{ineq:b}
    \text{(b)}= \alpha\mathbb{E}_{\textbf{x}\sim\pi^{\lambda}}[g^{\lambda}(\textbf{x})-g(\textbf{x})].
\end{equation}
Combining (\ref{ineq:a}), (\ref{ineq:b}), (\ref{ineq:c}) and (\ref{appendix:eqn1}), we have:
\begin{equation*}
\label{ineq:e}
    \alpha\mathbb{E}_{\textbf{x}\sim\pi}[g^{\lambda}(\textbf{x})-g(\textbf{x})] - \alpha\mathbb{E}_{\textbf{x}\sim\pi^{\lambda}}[g^{\lambda}(\textbf{x})-g(\textbf{x})]\geq \mu \mathcal{W}^{2}_{2}(\pi^{\lambda},\pi),
\end{equation*}
which induces the following inequality by using the fact that $g(\textbf{x}) \geq g^{\lambda}(\textbf{x})$ for all $\textbf{x}\in\mathbb{R}^{d}$:
\begin{equation}
\label{ineq:w2}
     \frac{\alpha}{\mu}\mathbb{E}_{\textbf{x}\sim\pi^{\lambda}}[g(\textbf{x}) - g^{\lambda}(\textbf{x})]\geq  \mathcal{W}^{2}_{2}(\pi^{\lambda},\pi).
\end{equation}
Most importantly, it is known that $g^{\lambda}(\textbf{x})$ converges pointwise to $g(\textbf{x})$ as $\lambda$ tends to zero \cite{rockafellar2009variational}. Therefore, using (\ref{ineq:w2}), $\mathcal{W}^{2}_{2}(\pi^{\lambda},\pi)$ tends to zero as $\lambda$ tends to zero.

We now prove (\ref{eqn:boundW2}). Given that $g$ is Lipschitz, by definition of $g^{\lambda}$ (\ref{def:envelope}), we have: for all $\textbf{x}\in \mathbb{R}^{d}$

\begin{align}
\label{ineqn:lipschitz}
\begin{split}
    g(\textbf{x}) - g^{\lambda}(\textbf{x}) =& g(\textbf{x}) - \inf_{\textbf{y}\in \mathbb{R}^{d}}\left\{ g(\textbf{y})+\frac{1}{2\lambda}\lVert \textbf{x}-\textbf{y} \rVert^{2}\right\}= \sup_{\textbf{y}\in \mathbb{R}^{d}}\left\{ g(\textbf{x}) - g(\textbf{y})-\frac{1}{2\lambda}\lVert \textbf{x}-\textbf{y} \rVert^{2}\right\}\\
    \leq & \sup_{\textbf{y}\in \mathbb{R}^{d}}\left\{ \lVert g \rVert_{\text{Lip}}\lVert \textbf{x}-\textbf{y}  \rVert-\frac{1}{2\lambda}\lVert \textbf{x}-\textbf{y} \rVert^{2}\right\} = \frac{\lambda }{2}\lVert g \rVert_{\text{Lip}}^{2}.
\end{split}
\end{align}
We conclude (\ref{eqn:boundW2}) by applying (\ref{ineqn:lipschitz}) to (\ref{ineq:w2}).

\end{proof}
\subsection{Proof of Theorem \ref{thm:saddle}}
\label{appendix:saddle}
\begin{proof}
In the same vein as VT \cite{liu2021infinite} and mirrorVT \cite{nguyen2023mirror}, our method is based on the assumption that $F$ admits the variational representation (\ref{eqn:varform}). Thus we replace $F(q)$ with its variational representation in problem (\ref{eqn:smoothregDP}), leading to the following concave-convex saddle point problem:

\begin{align}
\label{eqn:saddle}
    \min_{ q\in \mathcal{P}_{2}(\mathcal{X})} \max_{h\in\mathcal{H}}\left\{ \mathbb{E}_{\textbf{x}\sim q}\left[ h(\textbf{x})\right]-F^{*}(h) + \alpha \mathbb{E}_{\textbf{x}\sim q}\left[g^{\lambda}(\textbf{x}) \right]\right\}.
\end{align}

It is easy to verify that problem (\ref{eqn:saddle}) is concave-convex, so that the strong duality holds, which indicates that problem (\ref{eqn:saddle}) is equivalent to the following one:
\begin{align*}
\label{eqn:saddle2}
    \max_{h\in\mathcal{H}}\min_{ q\in \mathcal{P}_{2}(\mathcal{X})} \left\{ \mathbb{E}_{\textbf{x}\sim q}\left[ h(\textbf{x})+\alpha g^{\lambda}(\textbf{x}) \right]-F^{*}(h) \right\}
\end{align*}
It is noted that the inner minimization over the probability distribution space is computationally intractable. Consequently, we define $\mathcal{F}$ as the space of transport mapping functions that transform the simple distribution $p_{\epsilon}$ into $q\in \mathcal{P}_{2}(\mathcal{X})$. We further constrain $\mathcal{F}$ to be a subset of the square-integrable functions such that the minimization over $\mathcal{F}$ can be solved numerically, such as a class of deep neural networks $V(\cdot, \theta)$, parameterized by $\theta$. Let $V(\cdot, \theta^{*})$ represent the mapping function that transforms $p_{\epsilon}$ into $\pi^{\lambda}$, the optimal solution of problem (\ref{eqn:smoothregDP}). Assuming that $V(\cdot, \theta^{*})\in \mathcal{F}$, we obtain the following optimization problem through reparameterization:
\begin{align*}
    \max_{h\in\mathcal{H}}\min_{V(\cdot, \theta)\in \mathcal{F} } \left\{ \mathbb{E}_{\epsilon\sim p_{\epsilon}}\left[ h(V(\epsilon, \theta))+\alpha g^{\lambda}(V(\epsilon, \theta)) \right]-F^{*}(h) \right\}.
\end{align*}
Furthermore, since $h(\textbf{x})$ and $g^{\lambda}(\textbf{x})$ are continuous, we can apply the interchangeable principle from \cite{dai2017learning,rockafellar2009variational} to obtain the following equality:
\begin{align*}
    \min_{V(\cdot, \theta)\in \mathcal{F} } \left\{ \mathbb{E}_{\epsilon\sim p_{\epsilon}}\left[ h(V(\epsilon, \theta))+\alpha g^{\lambda}(V(\epsilon, \theta)) \right]\right\}=
    \mathbb{E}_{\epsilon\sim p_{\epsilon}}\left[ \min_{\textbf{x}_{\epsilon}\in\mathcal{X}} \left\{ h(\textbf{x}_{\epsilon})+ \alpha g^{\lambda}(\textbf{x}_{\epsilon}) \right\}  \right], 
\end{align*}
which concludes the proof.
\end{proof}

\subsection{Proof of Theorem \ref{thm:gradflow}}
\label{appendix:flowsolution}
We need the following lemma, known as the gradient dominance of the geodesically strongly convex functional, in order to show the convergence of $G^{\lambda}(q)$.
\begin{lem}
    \label{lem:gradientdominance}
    (Gradient Dominance) Let $F$ be a geodesically $\mu$-strongly convex and $F_{\texttt{opt}}=\min_{q\in\mathcal{P}_{2}(\mathcal{X})}F(q)$, we have:
    \begin{align*}
        2\mu \left[ F(q)-F_{\texttt{opt}}\right]\leq \langle \texttt{grad}F(q), \texttt{grad}F(q) \rangle_{q}.
    \end{align*}
\end{lem}
\begin{proof}
    Let $q_{\texttt{opt}}=\argmin_{q\in\mathcal{P}_{2}(\mathcal{X})}F(q)$, by the definition of geodesically strong convexity, we have:
    \begin{align*}
        &F(q_{\texttt{opt}})\geq F(q)+\langle \texttt{grad}F(q), \texttt{Exp}_{q}^{-1}(q_{\texttt{opt}})\rangle_{q} + \frac{\mu}{2}\mathcal{W}^{2}_{2}(q,q_{\texttt{opt}})\\
        & =F(q)+\langle \texttt{grad}F(q), \texttt{Exp}_{q}^{-1}(q_{\texttt{opt}})\rangle_{q} 
        + \frac{\mu}{2}\langle \texttt{Exp}_{q}^{-1}(q_{\texttt{opt}}), \texttt{Exp}_{q}^{-1}(q_{\texttt{opt}})\rangle_{q},
    \end{align*}
    where the equality is obtained by the definition of 2-Wasserstein distance: $\mathcal{W}^{2}_{2}(q,q_{\texttt{opt}})=\langle \texttt{Exp}_{q}^{-1}(q_{\texttt{opt}}), \texttt{Exp}_{q}^{-1}(q_{\texttt{opt}})\rangle_{q}$. We continue as follows:
    \begin{align*}
        F(q_{\texttt{opt}})\geq & F(q)+\frac{\mu}{2}\langle \texttt{Exp}_{q}^{-1}(q_{\texttt{opt}}) +\frac{1}{\mu}\texttt{grad}F(q), \texttt{Exp}_{q}^{-1}(q_{\texttt{opt}}) 
        +\frac{1}{\mu}\texttt{grad}F(q)\rangle_{q} - \frac{1}{2\mu}\langle \texttt{grad}F(q), \texttt{grad}F(q) \rangle_{q}\\
        \geq & F(q) - \frac{1}{2\mu}\langle \texttt{grad}F(q), \texttt{grad}F(q) \rangle_{q}.
    \end{align*}
    The lemma is proved.
\end{proof}

Now we prove Theorem \ref{thm:gradflow}. 
\begin{proof}
We consider the following gradient flow:
\begin{equation*}
    \frac{\partial \textbf{x}_{t}}{\partial t}=-v_{t}(\textbf{x}_{t})
\end{equation*}
and assume that $\textbf{x}_{0}$ is random with density $q_{0}$. As $q_{t}$ is the density of $\textbf{x}_{t}$, by directly applying the Fokker-Planck equation, we obtain:
\begin{align*}
    \frac{\partial q_{t}(\textbf{x})}{\partial t}= &\texttt{div}\left(  q_{t}(\textbf{x}) v_{t}(\textbf{x})\right)\\
    =& \texttt{div}\left(  q_{t}(\textbf{x})  \left[ \nabla h^{*}_{t}(\textbf{x})+\frac{\alpha}{\lambda}\left( \textbf{x}- \texttt{prox}^{\lambda}_{g}(\textbf{x})\right) \right]\right)\\
    = &\texttt{div}\left( q_{t}(\textbf{x}) \nabla \left[\frac{\partial F}{\partial q_{t}}(\textbf{x}) + \frac{\partial \mathbb{E}_{q_{t}}\left[g^{\lambda}\right]}{\partial q_{t}}(\textbf{x})\right]  \right)
    =\texttt{div}\left(q_{t}(\textbf{x}) \nabla \frac{\partial G^{\lambda}}{\partial q_{t}}(\textbf{x}) \right),
\end{align*}
where the second equality is obtained by the definition of $v_{t}$, the third equality is obtained by the facts $h_{t}^{*}=\partial F / \partial q_{t}$ and $g^{\lambda}=\partial \mathbb{E}_{q_{t}}\left[ g^{\lambda}\right] / \partial q_{t}$. The last equation shows that the continuity equation (\ref{eqn:continuityequation}) is the Wasserstein gradient flow of $G^{\lambda}(q)$ in the space of probability distributions with 2-Wasserstein metric. Lastly, we can easily show that:

\begin{align*}
\frac{\mathrm{d}}{\mathrm{d}t}\left(G^{\lambda}(q_{t})-G^{\lambda}(\pi^{\lambda}) \right)= &\int_{\mathcal{X}} \frac{\partial G^{\lambda}}{\partial q}(q_{t})(\textbf{x})\frac{\partial q_{t}}{\partial t}(\textbf{x})\mathrm{d}\textbf{x}\\
= &\int_{\mathcal{X}} \frac{\partial G^{\lambda}}{\partial q_{t}}(q_{t})(\textbf{x})\texttt{div}\left(q_{t}(\textbf{x})\nabla \frac{\partial G^{\lambda}}{\partial q}(q_{t})(\textbf{x})\right)\mathrm{d}\textbf{x}\\
= & - \int_{\mathcal{X}} \langle \nabla \frac{\partial G^{\lambda}}{\partial q}(q_{t})(\textbf{x}),\nabla \frac{\partial G^{\lambda}}{\partial q}(q_{t})(\textbf{x})\rangle q_{t}(\textbf{x})\mathrm{d}(\textbf{x})
=- \langle \texttt{grad}G^{\lambda}(q_{t}), \texttt{grad}G^{\lambda}(q_{t})  \rangle_{q_{t}},
\end{align*}
where the last equality is obtained by applying integration by parts. Since $G^{\lambda}$ is composed of two geodesically strongly convex functionals $F$ and $\mathbb{E}_{q}\left[g^{\lambda}\right]$, it is easy to verify that $G^{\lambda}$ is geodesically $\mu$-strongly convex. Using Lemma \ref{lem:gradientdominance}, we have:
\begin{align*}
    \frac{\mathrm{d}}{\mathrm{d}t}\left(G^{\lambda}(q_{t})-G^{\lambda}(\pi^{\lambda}) \right) \leq -2\mu \left( G^{\lambda}(q_{t})-G^{\lambda}(\pi^{\lambda}) \right).
\end{align*}
From a straight forward application of the Gronwall's inequality \cite{gronwall1919note}, it follows that:
\begin{align*}
    G^{\lambda}(q_{t})-G^{\lambda}(\pi^{\lambda})\leq 
    \exp(-2\mu t)( G^{\lambda}(q_{0})-G^{\lambda}(\pi^{\lambda})),
\end{align*}
which concludes the proof.
\end{proof}

\subsection{Additional Experiments on Synthetic Data}
\label{appendix:syntheticexperiments}
\noindent
\textbf{Experiment settings}. Similar to the experiments for the KL divergence, we compare MYVT and VT in terms of MSE and average $l_{1}$-norm for the first case study, and MSE and average TV semi-norm for the second case study. We set $\alpha=0.01$ and $\alpha=0.1$ for MYVT in the first and second case studies, respectively. We parameterize the function $V$ with a neural network with four layers, each of which has 100 neurons. For the function $h^\prime$, we use another neural network with two layers, each of which has 100 neurons, to parameterize it. To guarantee the outputs of $h^\prime$ to be in $(0,1)$, we use the sigmoid activation function in the last layer. The step sizes for VT and MYVT are set to 0.01 and 0.001, respectively. We run $K=2000$ iterations for the first case study and $K=4000$ iterations for the second case study.\\

\noindent
\textbf{Results}. The results for the first case study are illustrated in Figure \ref{fig:sparsity_JS}. Both methods MYVT and VT are able to generate samples that are similar to the truth, as indicated by the decreasing MSE over 2000 iterations (see Figure \ref{fig:sparsity_JS}a). However, MYVT keeps the average $l_{1}$-norm much lower than that of VT over iterations (see Figure \ref{fig:sparsity_JS}b). In particular, MYVT consistently produce samples with much lower average $l_{1}$-norm (around 16.35) compared to VT (around 31.81). Visually, samples generated by MYVT appear much sparser than those generated by VT (see Figure \ref{fig:sparsity_JS}c and \ref{fig:sparsity_JS}d). 

\begin{figure*}
    \centering
    \includegraphics[width=1.0\textwidth]{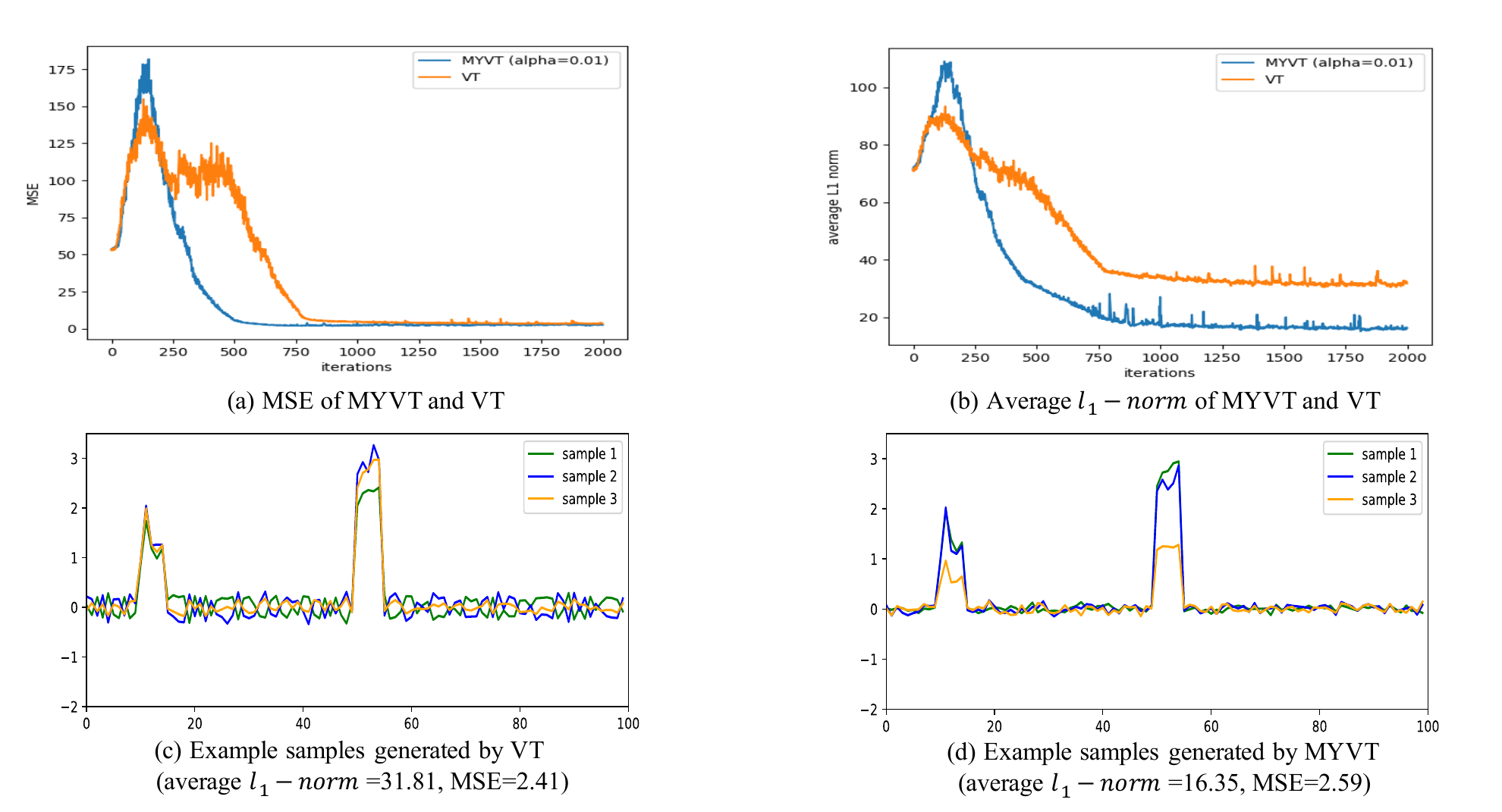}
    \caption{Comparison of MYVT($\alpha=0.01$) and VT in terms of MSE and sparsity (average $l_{1}$-norm of generated samples), when $F$ is JS divergence. (a) MSE of MYVT and VT over 2000 iterations, (b) average $l_{1}$-norm over 2000 iterations, (c) three example samples generated by VT, (d) three example samples generated by MYVT.}
    \label{fig:sparsity_JS}
\end{figure*}

\begin{figure*}
    \centering
    \includegraphics[width=1.0\textwidth]{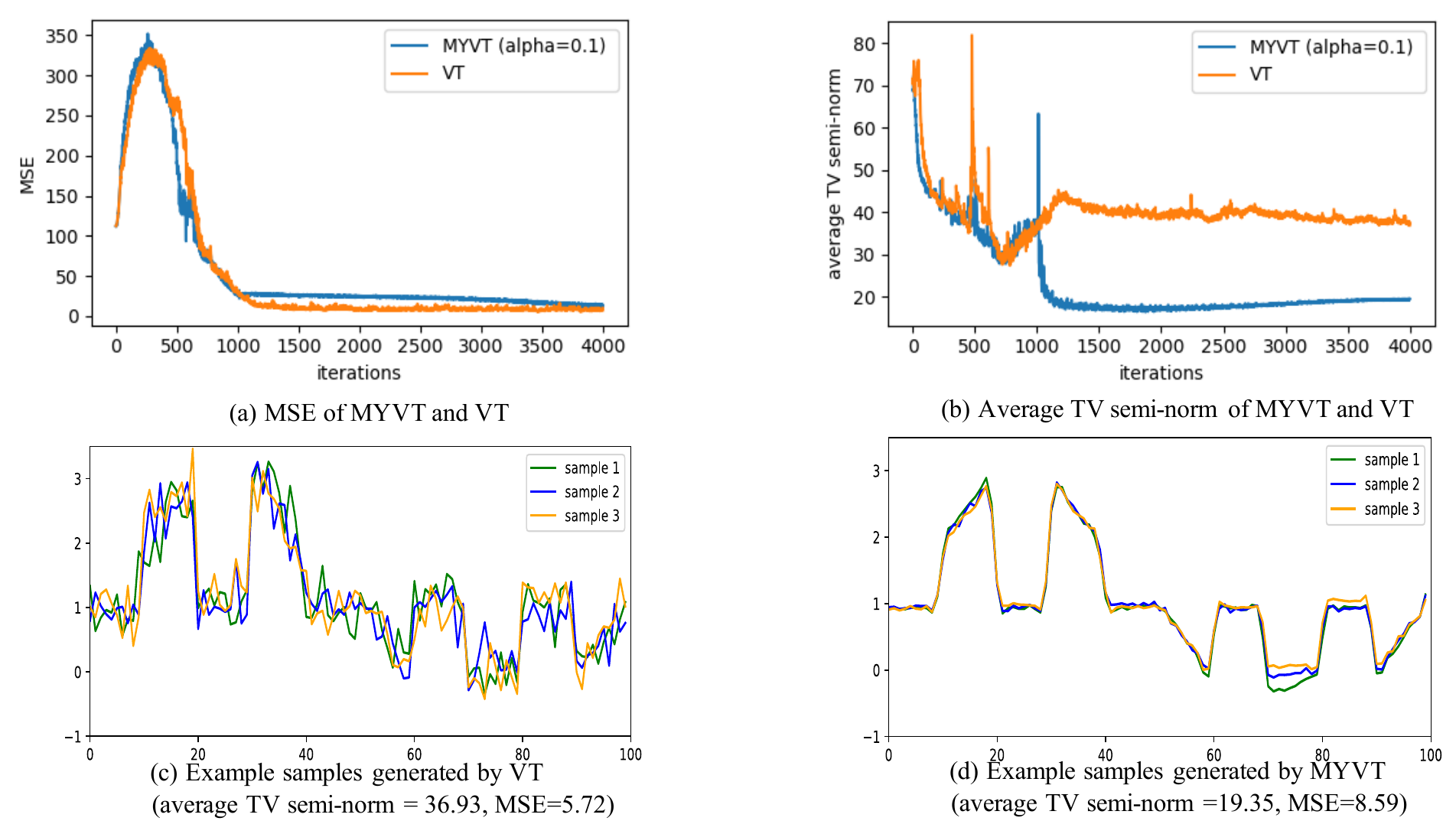}
    \caption{Comparison of MYVT($\alpha=0.1$) and VT in terms of MSE and sparsity (average $l_{1}$-norm of generated samples), when $F$ is JS divergence. (a) MSE of MYVT and VT over 4000 iterations, (b) average $l_{1}$-norm over 4000 iterations, (c) three example samples generated by VT, (d) three example samples generated by MYVT.}
    \label{fig:TV_JS}
\end{figure*}

The results for the second case study are shown in Figure \ref{fig:TV_JS}. The methods are again able to produce samples that are closely resemble the truth, as evidenced by the significant decrease in MSE values over 4000 iterations (see Figure \ref{fig:TV_JS}a). However, MYVT produces samples with a lower average TV semi-norm due to the effect of the TV regularization. Figures \ref{fig:TV_JS}c and \ref{fig:TV_JS}d show example samples generated by VT and MYVT, respectively. The average TV semi-norm of samples generated by MYVT (19.35) is significantly lower than that of samples generated by VT (36.93). These results in both case studies confirm the regularization effects of problem (\ref{eqn:regDP}) on the generated samples and effectiveness of MYVT in the case of JS divergence.

\label{syntheticimage}
\begin{figure}
\centering
    \begin{subfigure}{.25\textwidth}
        \center
        \includegraphics[width=1.0\linewidth]{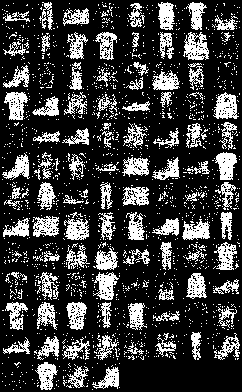}
        \caption{WGAN}
        \label{subfig:11}
    \end{subfigure}
     \begin{subfigure}{.25\textwidth}
        \center
        \includegraphics[width=1.0\linewidth]{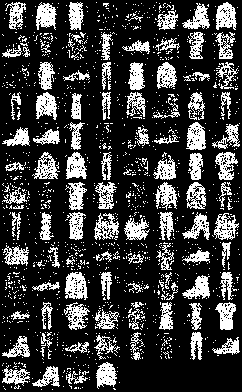}
        \caption{InfoGAN}
        \label{subfig:12}
    \end{subfigure}
    \begin{subfigure}{.25\textwidth}
        \center
        \includegraphics[width=1.0\linewidth]{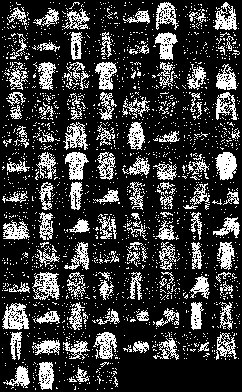}
        \caption{GAN}
        \label{subfig:13}
    \end{subfigure}

    \centering
    \begin{subfigure}{.25\textwidth}
        \centering
        \includegraphics[width=1.0\linewidth]{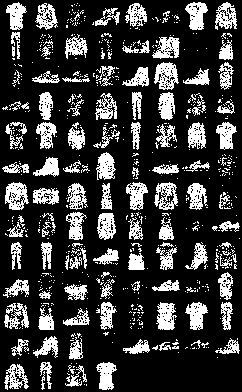}
        \caption{MYVT ($\alpha=10^{-4}$)}
        \label{subfig:14}
    \end{subfigure}
    \begin{subfigure}{.25\textwidth}
        \center
        \includegraphics[width=1.0\linewidth]{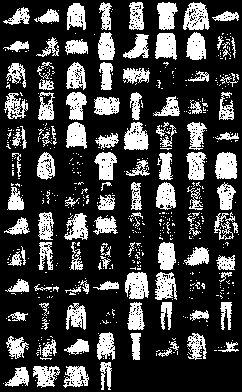}
        \caption{MYVT ($\alpha=10^{-3}$)}
        \label{subfig:15}
    \end{subfigure}

    \caption{Synthetic images generated by WGAN (a), InfoGAN (b), GAN (c) and MYVT ((d) and (e)) trained on noisy training examples of FashionMNIST at noise level $\sigma=0.5$.}
    \label{fig:syntheticimages0.5_fashion}
\end{figure}

\subsection{Additional Experiments on Real-world Datasets}
Figure \ref{fig:syntheticimages0.5_fashion} and Figure \ref{fig:syntheticimages0.1_fashion} showcase the quality of synthetic images generated by WGAN, InfoGAN, GAN and MYVT trained on noisy images of FashionMNIST at noise levels $\sigma=0.5$ and $0.1$, respectively. Similarly, Figure \ref{fig:syntheticimages0.5_caltech} and Figure \ref{fig:syntheticimages0.1_caltech} demonstrates the quality of synthetic images generated by the models trained on noisy images of Caltech101Silhouettes at the noise levels. The visualizations reveal that WGAN, InfoGAN and GAN struggle to generate clean images due to the presence of noisy training images, resulting in images with artifacts. In contrast, MYVT can overcome this problem thanks to the regularization part in the objective function. These results reinforces the effectiveness of the proposed method, MYVT.

\begin{figure}
\centering
    \begin{subfigure}{.25\textwidth}
        \center
        \includegraphics[width=1.0\linewidth]{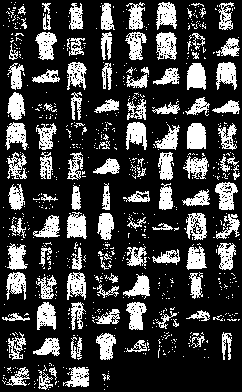}
        \caption{WGAN}
        \label{subfig:16}
    \end{subfigure}
    \begin{subfigure}{.25\textwidth}
        \center
        \includegraphics[width=1.0\linewidth]{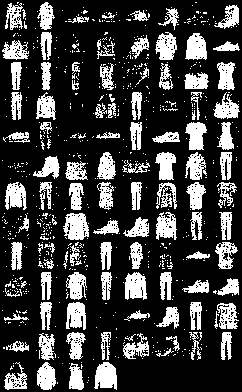}
        \caption{InfoGAN}
        \label{subfig:17}
    \end{subfigure}
    \begin{subfigure}{.25\textwidth}
        \center
        \includegraphics[width=1.0\linewidth]{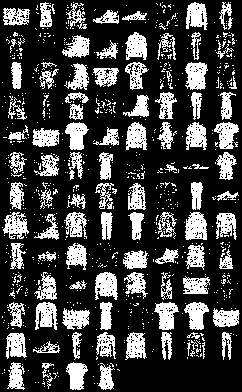}
        \caption{GAN}
        \label{subfig:18}
    \end{subfigure}

    \centering
    \begin{subfigure}{.25\textwidth}
        \centering
        \includegraphics[width=1.0\linewidth]{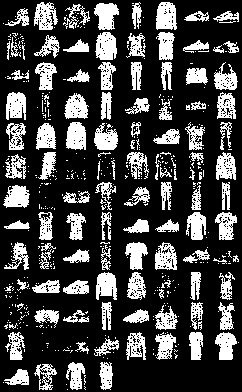}
        \caption{MYVT \\($\alpha=10^{-4}$)}
        \label{subfig:19}
    \end{subfigure}
    \begin{subfigure}{.25\textwidth}
        \center
        \includegraphics[width=1.0\linewidth]{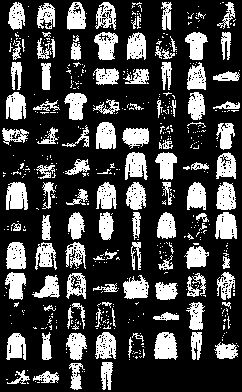}
        \caption{MYVT \\($\alpha=10^{-3}$)}
        \label{subfig:20}
    \end{subfigure}

    \caption{Synthetic images generated by WGAN (a), InfoGAN (b), GAN (c) and MYVT ((d) and (e)) trained on noisy training examples of FashionMNIST at noise level $\sigma=0.1$.}
    \label{fig:syntheticimages0.1_fashion}
\end{figure}

\begin{figure}
\centering
    \begin{subfigure}{.25\textwidth}
        \center
        \includegraphics[width=1.0\linewidth]{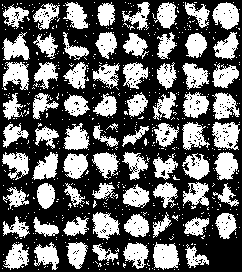}
        \caption{WGAN}
        \label{subfig:21}
    \end{subfigure}
    \begin{subfigure}{.25\textwidth}
        \center
        \includegraphics[width=1.0\linewidth]{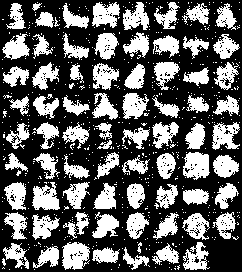}
        \caption{InfoGAN}
        \label{subfig:22}
    \end{subfigure}
    \begin{subfigure}{.25\textwidth}
        \center
        \includegraphics[width=1.0\linewidth]{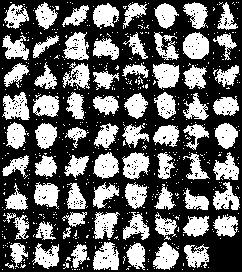}
        \caption{GAN}
        \label{subfig:23}
    \end{subfigure}

    \centering
    \begin{subfigure}{.25\textwidth}
        \centering
        \includegraphics[width=1.0\linewidth]{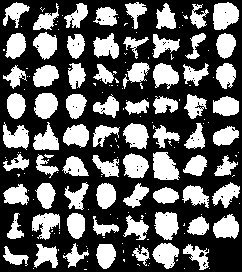}
        \caption{MYVT \\($\alpha=10^{-4}$)}
        \label{subfig:24}
    \end{subfigure}
    \begin{subfigure}{.25\textwidth}
        \center
        \includegraphics[width=1.0\linewidth]{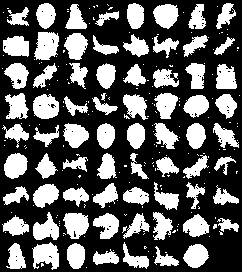}
        \caption{MYVT \\($\alpha=10^{-3}$)}
        \label{subfig:25}
    \end{subfigure}

    \caption{Synthetic images generated by WGAN (a), InfoGAN (b), GAN (c) and MYVT ((d) and (e)) trained on noisy training examples of Caltech101Silhouettes at noise level $\sigma=0.5$.}
    \label{fig:syntheticimages0.5_caltech}
\end{figure}

\begin{figure}
\centering
    \begin{subfigure}{.25\textwidth}
        \center
        \includegraphics[width=1.0\linewidth]{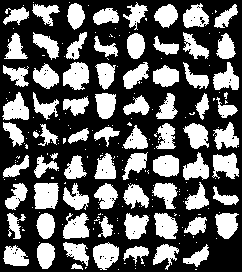}
        \caption{WGAN}
        \label{subfig:26}
    \end{subfigure}
     \begin{subfigure}{.25\textwidth}
        \center
        \includegraphics[width=1.0\linewidth]{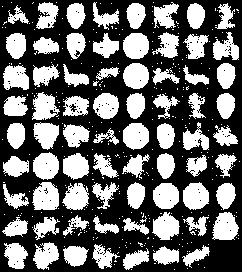}
        \caption{InfoGAN}
        \label{subfig:27}
    \end{subfigure}
    \begin{subfigure}{.25\textwidth}
        \center
        \includegraphics[width=1.0\linewidth]{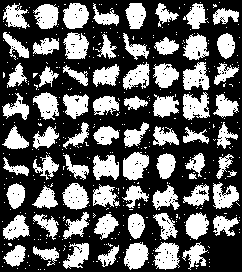}
        \caption{GAN}
        \label{subfig:28}
    \end{subfigure}

    \centering
    \begin{subfigure}{.25\textwidth}
        \centering
        \includegraphics[width=1.0\linewidth]{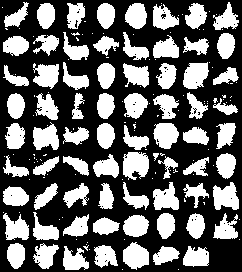}
        \caption{MYVT ($\alpha=0.0001$)}
        \label{subfig:29}
    \end{subfigure}
    \begin{subfigure}{.25\textwidth}
        \center
        \includegraphics[width=1.0\linewidth]{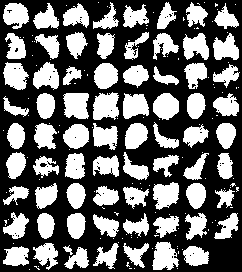}
        \caption{MYVT ($\alpha=0.001$)}
        \label{subfig:30}
    \end{subfigure}

    \caption{Synthetic images generated by WGAN (a), InfoGAN (b), GAN (c) and MYVT ((d) and (e)) trained on noisy training examples of Caltech101Silhouettes at noise level $\sigma=0.1$.}
    \label{fig:syntheticimages0.1_caltech}
\end{figure}
\end{document}